\documentclass{amsart}

\usepackage{amsmath}
\usepackage{amssymb}
\usepackage{amsfonts}
\usepackage{mathtools}
\usepackage{amsthm}
\usepackage{amsfonts}
\usepackage{mathrsfs}
\usepackage{subcaption}
\usepackage{xcolor}
\usepackage{float}
\usepackage{enumitem}
\usepackage{placeins}
\usepackage{algorithm}
\usepackage{algpseudocode}
\usepackage{hyperref}
\usepackage{multirow}

\usepackage[backend=biber]{biblatex}
\newtheorem{theorem}{Theorem}[section]
\newtheorem{proposition}[theorem]{Proposition}
\newtheorem{lemma}[theorem]{Lemma}

\theoremstyle{definition}
\newtheorem{definition}[theorem]{Definition}
\newtheorem{assumption}[theorem]{Assumption}
\theoremstyle{remark}
\newtheorem{remark}[theorem]{Remark}

\newcommand{\E}{{\mathbb E}}

\newcommand{\R}{{\mathbb R}}
\renewcommand{\L}{\mathscr{L}}

\newcommand{\norm}[1]{\left\|#1\right\|}

\newcommand{\C}{\mathcal C}
\newcommand{\N}{\mathcal N}

\newcommand{\dd}{{\mathrm d}}
\newcommand{\fabian}[1]{{#1}}

\definecolor{myDarkBlue}{RGB}{0, 0, 139}
\newcommand{\lam}[1]{\textcolor{myDarkBlue}{#1}}

\renewcommand{\lam}[1]{{#1}}
\newcommand{\revise}[1]{{#1}}

\newcommand{\mua}{\mu_{{\mathrm a}}}  
\newcommand{\mus}{\mu'_{{\rm s}}}
\newcommand{\mm}{ {\rm mm}}

\newlength{\globaltextwidth}
\AtBeginDocument{\setlength{\globaltextwidth}{\textwidth}}

\title[SBD for severely ill-posed problems in DOT]{Score-based diffusion models for severely ill-posed problems in diffuse optical tomography}

\thanks{$^*$ Corresponding author. $^1$ Department of Computational Engineering, School of Engineering Science LUT University, Finland. \texttt{tapio.helin@lut.fi}\\
$^2$ Institute of Analysis and Scientific Computing, TU Wien, Austria. \texttt{fabian.schneider@asc.tuwien.ac.at, leila.taghizadeh@asc.tuwien.ac.at}\\
$^3$ Department of Technical Physics, Faculty of Science, Forestry and Technology, University of Eastern Finland, Finland. \texttt{meghdoot.mozumder@uef.fi, konstantin.tamarov@uef.fi, tanja.tarvainen@uef.fi}\\
$^4$ Research Unit of Mathematical Sciences, University of Oulu, Finland. \texttt{duc-lam.duong@oulu.fi}.}

 \author[F.S]{Fabian Schneider$^{1, 2, *}$}
\author[M.M.]{Meghdoot Mozumder$^3$}
\author[K.T]{Konstantin Tamarov$^3$}
\author[L.T.]{Leila Taghizadeh$^2$}
\author[T.T.]{Tanja Tarvainen$^3$}
\author[T.H.]{Tapio Helin$^1$}
\author[D-.L.D.]{Duc-Lam Duong$^{1,4}$}
\begin{document}


\maketitle

\begin{abstract}
Score-based diffusion models are a recently developed framework for posterior sampling in Bayesian inverse problems, enabling high-quality reconstructions in inverse problems by leveraging expressive prior distributions learned from empirical data. Despite their strong empirical performance and growing interest within the machine learning community, their behaviour in realistic, severely ill-posed inverse problems with experimental measurement data remains under-explored.
Diffuse optical tomography (DOT) is an inverse boundary value problem that uses boundary measurements of near-infrared light to recover spatially varying absorption and scattering parameters in biological tissue. The problem is highly ill-posed and particularly sensitive to both measurement noise and modelling errors.

We introduce a regularization strategy by constructing a mixed score consisting of a learned component and a model-based component. \fabian{We show that the resulting mixed score approximates the score of a corresponding mixture distribution locally and in the small diffusion-time regime, providing a theoretical justification for the approach.}

We compare four approaches for difference imaging in DOT: a classical model-based method, an approximate score-based diffusion method (DPS), an exact posterior sampling method (UCoS) and a novel, regularized version of UCoS. 
We show that both the model-based approach and approximate diffusion-based sampling degrade significantly in the presence of limited-view geometry and real experimental data, whereas UCoS yields more accurate reconstructions.
\end{abstract}

\section{Introduction}

Score-based diffusion (SBD) models provide a framework for sampling from complex probability distributions by learning from empirical data, enabling high-quality image generation with state-of-the-art performance \cite{sohldicksteindeep15, Ho2020DenoisingDP, Song2020ScoreBasedGM}.
By training on expert-generated image datasets, these models capture highly complex distributions, often outperforming classical prior models that rely on handcrafted structures and tuned parameters.

In the context of Bayesian inverse problems, diffusion-based priors have recently gained significant attention due to their strong empirical performance \cite{song2022solving, chung2023diffusion, dou2024diffusion, baldassari2023conditional, feng2023scorebaseddiffusionmodelsprincipled, schneider2025scalablediffusionposteriorsampling}.
They enable posterior sampling, which facilitates comprehensive uncertainty quantification, including estimates of posterior mean and covariance.

These methods have been applied to a range of inverse problems, including computed tomography (CT) \cite{song2022solving, schneider2025scalablediffusionposteriorsampling}, magnetic resonance imaging (MRI) \cite{song2022solving}, photoacoustic tomography \cite{Dey2024SBDforPAT}, and seismic imaging \cite{baldassari2023conditional}. 
Crucially, three aspects remain under-explored in the literature. 
First, highly ill-posed inverse problems, such as inverse boundary value problems, have received comparatively little attention. In such problems the measurement provides relatively little information and the reconstructions are heavily influenced by the prior distribution. It is thus crucial to utilize an expressive prior and minimize errors through approximate conditioning of the score function.
Second, most existing studies rely on simulated measurement data, rather than data acquired from physical measurement systems.
The impact of approximation errors due to imperfect mathematical models and assumptions are a fundamental component of solving realistic inverse problems and cannot be neglected when assessing the performance of learned reconstruction methods.
Finally, most studies assume a perfect match between the training and evaluation data distributions. In practice, however, distributional mismatch, limited training data, and expert information not captured in finite training data, such as tail behaviour or physical constraints, can all hinder the performance of learned score models.

Diffuse optical tomography (DOT) utilises boundary measurements of near-infrared light to estimate spatially distributed absorption and scattering parameters in biological tissues \cite{arridge1999optical, arridge2009optical, Aspri2024}. Mapping these optical parameters provides valuable information about tissue function and structure, with applications in breast cancer imaging, neonatal brain imaging, functional imaging of the adult brain, and preclinical small animal imaging \cite{durduran2015,hoshi2016overview}.

The image reconstruction problem in diffuse optical tomography (DOT) is a non-linear and highly ill-posed inverse problem that involves estimating the spatial distribution of optical absorption and scattering parameters \cite{arridge1999optical,arridge2009optical}. A linearized approximation based on a first-order Taylor expansion of the forward model (e.g., Born or Rytov approximation) is often employed \cite{arridge2009optical}. In the linearized approach,  changes in measured signals are considered as a linear function of the changes in the optical parameters, and is an ill-posed linear inverse problem. Given the ill-posedness of the DOT image reconstruction, regularization techniques that assume the unknowns as smooth, sparse, or with limited changes, are typically utilized to achieve stable inversion \cite{Aspri2024}. In a similar way, Bayesian estimation incorporates prior probability distributions based on existing knowledge about the unknowns, to compute the posterior probability distribution, providing a measure of the estimates' uncertainty \cite{kaipio2006statistical, stuart2010inverse}. 

Recently, deep learning (DL) has emerged as a data-driven alternative to traditional model-based DOT reconstructions. 
One of the first attempts towards deep learned image reconstruction was carried out in \cite{ben2018deep} using a neural network comprised of fully connected and convolution layers. 
Yoo {\em et al.} \cite{yoo2019deep} used convolutional framelet theory to demonstrate that a neural network can represent the DOT image reconstruction, and used a similar network architecture to estimate difference absorption coefficients. The approach was extended to use a separately trained U-Net, providing improved estimates \cite{deng2023fdu}. 
As an alternative to these direct learned reconstructions, unrolled deep networks have also been proposed for DOT \cite{mozumder2021model,liu2020learnable,yi2024enhanced}. 
In this approach, an iterative optimization algorithm is explicitly unrolled into a neural network architecture such that each iteration of the original algorithm corresponds to a layer (or block) of the network. 
Other approaches used co-registered ultrasound \cite{zou2021machine} or MRI \cite{yu2025high} images to train neural networks, by incorporating a total variation-based loss between the optical parameters and the co-registered image. 

Nevertheless, acceptance of DL techniques in practical applications has often been restricted by opaque predictions and unreliable confidence estimates. To address this, uncertainty quantification (UQ) techniques have been proposed in deep learned imaging techniques, that can estimate confidence bounds and highlight model failures or weaknesses \cite{lambert2024trustworthy}. 
The Bayesian framework, where all the unknowns and measured quantities are considered as random variables, is a natural choice for UQ.
The estimated posterior distribution can be used to compute \revise{credible} intervals, and thereby quantify the uncertainty in the estimates. In practice, however, exact Bayesian inference is rarely feasible for deep networks because of its computational demands. This gap is overcome by diffusion models, which are able to compute many samples in a feasible time frame even for large scale problems \cite{song2022solving, baldassari2023conditional, schneider2025scalablediffusionposteriorsampling}. To the best of our knowledge, diffusion models have not yet been applied to the DOT problem.

\subsection{Our contribution}

In this work, we consider the severely ill-posed inverse boundary problem in diffuse optical tomography.

To help to encode expert information into the distribution which is not necessarily well-presented in training data for example due to limited or imperfect training data, we introduce a novel framework to interpolate between two score functions, for example learned score and a model-based score. The approach is motivated as a non-centred version of the Tikhonov denoising score-matching proposed by \cite{Baptista2025MemorizationAR}. 
It corresponds to a convex combination of the score functions and we show that this construction locally approximates, in the small diffusion-time regime, the score of a geometric mixture of the corresponding probability density functions (PDFs). Crucially, the mixing parameter that interpolates the two score functions can be tuned online and without retraining of an underlying score function.

We consider four approaches for posterior sampling in the DOT problem: 
a classical model-based reconstruction method, the approximate score-based diffusion approach Diffusion posterior sampling (DPS) \cite{chung2023diffusion}, the recently proposed unconditional representation of the conditional score function (UCoS) framework \cite{schneider2025scalablediffusionposteriorsampling} and a regularized version of UCoS.

The classical model-based approach, based on a Gaussian prior, serves as a baseline grounded in explicit regularization and analytical modelling. 
DPS is a well studied method for posterior sampling in SBD. It introduces additional approximations in the sampling procedure, which may affect accuracy in challenging inverse settings and requires repeated evaluations of the forward map during sampling, both classical challenges for many approaches in SBD for inverse problems.
UCoS serves as an exact posterior sampling method for large-scale linear inverse problems. It is designed to enable efficient and accurate sampling by shifting the computational burden from inference to training. In particular, it avoids repeated evaluations of the forward operator during sampling, eliminates the need for an explicit matrix representation of the linear forward map, and maintains a tractable learning problem in a reduced-dimensional setting.

We provide numerical results on experimental data, and simulations covering full-view and limited-view geometries as well as in- and out-of-distribution targets, thereby directly addressing the gap between idealized simulation studies and real measurement settings highlighted in the literature.

We demonstrate that the approximate diffusion method and the model-based approaches degrade significantly under limited-view and experimental conditions, whereas UCoS-based sampling yields more accurate reconstructions of inclusion structure under these settings.
Across all experiments, UCoS recovers inclusion locations and amplitudes the most reliably, avoiding artefacts that arise from approximate conditioning of the score function. The proposed regularized UCoS variant further improves robustness in limited-view and distribution-shift scenarios, while maintaining competitive performance on experimental data.

\lam{
\section{Preliminaries} \label{sec:preliminaries}
\subsection{Bayesian inverse problems}
Consider a linear inverse problem
\begin{equation}
	\label{eq:ip}
	 y = Ax + \epsilon,
\end{equation}
where $y \in \R^m$ is the observed data, $x\in H$ is the unknown taking values in a separable Hilbert space $H$, equipped with inner product $\langle \cdot, \cdot \rangle$ and norm $\|\cdot\|_H$, and $A: H \rightarrow \R^m$ is the linear forward map. We denote by $A^*: \R^m \rightarrow H$ the adjoint of $A$.
We adopt the Bayesian framework \cite{stuart2010inverse} for the inverse problem \eqref{eq:ip}. In this framework, prior information on the unknown $x$ is encoded through a probability measure $\mu$ supported on $H$. 
The observational noise in \eqref{eq:ip} is assumed to be Gaussian, $\epsilon \sim \N(0, {\Gamma_{\text{obs}}})$,  where ${\Gamma_{\text{obs}}} \in \R^{m \times m}$ is symmetric positive definite.
Under these assumptions, the posterior measure of $x$ given data $y$ is absolutely continuous with respect to the prior measure $\mu$, with Radon-Nikodym derivative given by, $\mu$-almost everywhere, 
\begin{equation}
    \label{eq:posterior}
     \frac{\dd \mu^y}{\dd\mu}(x)  =  \frac{1}{Z(y)} \exp\left(-\frac 12 \norm{Ax -y}_{\Gamma_{\text{obs}}}^2\right),
\end{equation}
where
\begin{equation*}
    Z(y)  =  \int_H \exp\left(-\frac 12 \norm{Ax -y}_{\Gamma_{\text{obs}}}^2\right) \mu(\dd x)
\end{equation*}
is the normalising constant. We also use the weighted norm notation $\norm{v}_C^2 := \langle C^{-1} v, v \rangle$, where $C$ is an appropriately defined symmetric positive matrix (or operator) on $\R^m$ (or $H$), throughout the paper.}

\lam{In Bayesian inverse problems, one of the objectives is to sample from the posterior $\mu^y$. However, sampling from high-dimensional problems is often \revise{computationally} challenging in practice, especially when the underlying forward problem is governed by complicated PDEs such as the DOT problem. In the following sections, we first introduce the DOT problem together with its governing PDEs and measurement model (Section \ref{subsec:DOT}). The classical Gaussian model-based approach is introduced in Section \ref{subsec: modelbased}. Finally in Section \ref{subsec:SBD-sampling} we briefly review the important elements of the posterior sampling strategies using score-based diffusion models, including the conditional approach, unconditional diffusion posterior sampling (DPS), and UCoS framework.}

\subsection{Diffuse optical tomography}
\label{subsec:DOT}

In a typical DOT measurement setup, near-infrared light is introduced into an object from its boundary. Let $\Omega \subset \mathbb{R}^d, \, (d=2 \, \rm{or} \, 3)$ denote the bounded domain with boundary $\partial \Omega$ where $d$ is the spatial dimension of the domain. In a diffusive medium like soft biological tissue, the commonly used light transport model is the diffusion approximation to the radiative transfer equation \cite{arridge1999optical, arridge2009optical}. Here, we consider the frequency-domain version of the diffusion approximation
\begin{align}
\begin{aligned}\label{deeqn}
 \left(-\nabla \cdot \frac{1}{{d}(\mua(r)+\mus(r))} \nabla  + \mua(r) + \frac{{\rm j}\omega}{c} \right) \Phi(r) = 0, \hspace{0.1mm} r \in \Omega \\
\Phi(r)+\frac{1}{2\zeta}\frac{1}{d(\mua (r)+\mus(r))} \alpha \frac{\partial \Phi(r)}{\partial \hat{n}} = \left\{ \begin{array}{ll}
         \frac{q}{\zeta}, & r \in s\\
         0, & r \in \partial \Omega \setminus s \end{array} \right. ,
\end{aligned}
\end{align}
where $\Phi(r)$ is the photon fluence, $\mua(r)$ is the absorption coefficient, $\mus(r)$ is the (reduced) scattering coefficient, j is the imaginary unit and $c$ is the speed of light in the medium. The parameter $q$ is the strength of the light source at location $s \subset \partial\Omega$, operating at angular modulation frequency $\omega$. Further, the parameter $\zeta$ is a dimension-dependent constant ($\zeta$ = $1/\pi$ when $\Omega \subset \mathbb{R}^{2}$, $\zeta$ = $1/2$ when $\Omega \subset \mathbb{R}^{3}$) and $\alpha$ is a parameter governing the Fresnel reflection at the boundary $\partial \Omega$, and $\hat{n}$ is an outward unit vector normal to the boundary. The measurable data on the boundary of the object, exitance $\Gamma(r)$, is given by
\begin{equation}\label{bdd}
\Gamma(r) = -\frac{1}{d(\mua (r)+\mus(r))} \frac{\partial \Phi(r)}{\partial \hat{n}} = \frac{2\zeta}{\alpha}\Phi(r). 
\end{equation}

Let $H$ be a Hilbert space that supports the infinite dimensional optical parameters, $(\mua,\mus) \in H^2$. By $\mathcal{A}$ denote the forward operator that maps the optical parameters to the finite dimensional boundary data $\Gamma \in \R^m$. For frequency-domain DOT, the measured data typically consist of the logarithm of the complex boundary flux, written as
\begin{equation}\label{meas}
Y = \left(\begin{array}{c}
{\rm Re} \log(\Gamma)\\
{\rm Im} \log(\Gamma)
\end{array} \right) = \left(\begin{array}{c}
\log({\rm Amplitude})\\
{\rm Phase}
\end{array} \right),
\end{equation}
so that the forward model reads
\begin{equation}\label{infinite_forward}
Y = \mathcal{A}(\mua, \mus) + E,
\end{equation}
where $E$ models the additive random noise in measurements.
Consider two measurements $Y_{1}$ and $Y_{2}$ corresponding to optical parameters $X_1=$ ($\mu^{}_{{\rm a},1},\mu'_{{\rm s},1}$) and $X_2=$ ($\mu^{}_{{\rm a},2},\mu'_{{\rm s},2})$, respectively. The perturbation in the data due to changes in the parameters, $\delta Y = Y_2 - Y_1$, is given by the linear perturbation model \cite{arridge2009optical}
\begin{equation}\label{modeldiff}
      \delta Y = J \delta X + e,
    \end{equation}
where the Jacobian $J=(\frac{\partial \mathcal{A}}{\partial \mua}, \frac{\partial \mathcal{A}}{\partial \mus})$ denotes the Fr{\'e}chet derivative of the forward operator $\mathcal{A}$, and $e$ represents the measurement noise in the difference data. The linearization (\ref{modeldiff}) is based on a first-order Taylor approximation and it is widely considered robust for biological targets where variations in optical properties are relatively small.

Considering a typical range of difference absorption and scattering parameters $\delta\mua \in [-0.01,\: 0.01] {\rm mm}^{-1}$ and $\delta\mus \in [-1,\: 1] {\rm mm}^{-1}$ \cite{jacques2013optical}, we define the following affine transforms
\begin{eqnarray*}
    \tilde{\delta \mua}(r) = \frac{\delta\mua(r) + 0.01 }{0.02},\;\tilde{\delta \mus}(r) = \frac{\delta\mus(r) + 1}{2}, \qquad r \in \Omega
\end{eqnarray*}
where the rescaled optical parameters $x = (\tilde{\delta \mua}, \tilde{\delta \mus})$ lie in the range $[0,\: 1]$.
The linear model is obtained by substituting these into the model (\ref{modeldiff})
\begin{equation}\label{modeldiffnorml}
	 y = Ax + \epsilon,
\end{equation}
where $A = (0.02\frac{\partial \mathcal{A}}{\partial \mua}, 2\frac{\partial \mathcal{A}}{\partial \mus})$ is the scaled Jacobian, $x \in H, y \in \R^m$ and the noise term $\varepsilon \in \R^m$ is given by an affine transformation of the measurement noise $e$ in \eqref{modeldiff}.

\lam{In DOT, due to the diffusive nature of light propagation in biological tissue which strongly attenuates information about the interior optical parameters, the boundary measurements are heavily smoothed and blurred. The resulting inverse problem \eqref{modeldiffnorml} is thus severely ill-posed, making posterior sampling and uncertainty quantification particularly challenging.}

\subsection{Bayesian `model-based' estimation with a Gaussian prior}\label{subsec: modelbased}
\lam{In this section, we consider a Gaussian prior distribution $\nu = \N(m,S)$ associated \revise{with} the Ornstein-Uhlenbeck (OU) random process \cite{rasmussen2006}. The covariance function of the OU process is given by}
\begin{equation}
\label{eq:prior_covariance}
s(r_1, r_2) = \sigma^2 \exp( - \|r_1 - r_2\|  / \ell )
\end{equation}
where $r_1, r_2 \in \Omega \subset \R^d$ denote the spatial locations. Here, $\sigma$ stands for the marginal standard deviation and $\ell$ describes the characteristic length scale of the process, i.e., the spatial distance that the parameter is expected to have a significant spatial correlation. Moreover, we assume that the domain of interest $\Omega$ is compact.

Since $\Omega$ is compact, the Ornstein-Uhlenbeck covariance kernel $s$ belongs to $L^2(\Omega\times \Omega)$, and the associated integral operator $S:L^2(\Omega)\to L^2(\Omega)$ is self-adjoint, positive, and trace class \revise{(see \cite[Section 4.3]{rasmussen2006})}.  
Consequently, there exists a centered Gaussian random variable taking values in $L^2(\Omega)$ whose covariance operator is $S$, i.e., the Ornstein-Uhlenbeck random field admits a well-defined $L^2(\Omega)$-valued random variable version.
Notice carefully that in this section, for convenience, $\nu$ stands for the probability distribution of a single random function on $\Omega\subset \R^d$. However, in later sections we write $\nu$ for the distribution of the independent product of two random functions, namely, absorption and scattering coefficients, on $\Omega$.

\lam{
Under the Gaussian prior $\nu = \N(m, S)$ and the additive Gaussian observational noise $\epsilon \sim \N(0, {\Gamma_{\text{obs}}})$, the posterior corresponding to the linear inverse problem \eqref{eq:ip} is Gaussian, $\nu^y = \N(\bar{x}, \Gamma_{\text{post}})$, with mean and covariance
\begin{align}
\begin{aligned}\label{eq: gaussian posterior}
    \bar{x} &= m + S A^* \left( A S A^* + \Gamma_{\text{obs}}\right)^{-1}(y-Am), \\
\Gamma_{\text{post}} &= S - \fabian{S A^* (A S A^* + \Gamma_{\text{obs}})^{-1}A S},
\end{aligned}
\end{align}
see \cite[Theorem 6.20]{stuart2010inverse}. 
In the discretized setting $H = \R^n$ with non-degenerate covariance operator $S$, the posterior mean $\bar{x}$ coincides with the \emph{maximum a posteriori estimate} and is characterised by the Tikhonov-regularized optimization problem
\begin{equation}
\label{eq:map_gaussian_compact}
\bar{x} = \operatorname*{\arg \, \min}_{x} 
\left[\| y - A x\|^2_{\Gamma_{\text{obs}}} + \| \fabian{x -m}\|^2_{S}\right].
\end{equation}
For further details, we refer to \cite{stuart2010inverse, kaipio2006statistical}. 
In the following, we refer to this approach as \emph{model-based} since it relies solely on the mathematical model for the forward map and the prior distribution based on expert knowledge rather than  being data-driven.} \revise{We also note that, being a Gaussian distribution, the posterior arising from this approach admits a \emph{closed-form diffusion score}. This score forms the model-based component of the regularized UCoS method \revise{developed} in Section \ref{sec:reg-UCoS}.}

\subsection{Posterior sampling using score-based diffusion models}\label{subsec:SBD-sampling}

\lam{Score-based diffusion methods for inverse problems can broadly be divided into two classes.} 
\emph{Conditional} methods utilize the score function conditional on the measurement (i.e. the score function of the posterior), whereas \emph{unconditional} methods use the unconditional (prior) score function and approximately steer samples towards the posterior distribution using the measurement information and forward map during sampling.
\lam{For the convenience of the reader, we briefly summarize below both approaches and the UCoS framework, which provide the main ingredients for understanding the regularized UCoS method we develop in Section \ref{sec:reg-UCoS}. For detailed derivations we refer the reader to the original papers \cite{baldassari2023conditional, chung2023diffusion, schneider2025scalablediffusionposteriorsampling}. See also \cite{Song2020ScoreBasedGM} for the foundational ideas of SBD models.}

\subsubsection{Conditional score-based diffusion} \label{subsec:ConditionalSBD}
The mathematical setup of conditional score-based diffusion models is established in \cite{baldassari2023conditional}. The framework utilizes a stochastic differential equation (SDE) that transports the measure $\mu^y$ to $\N(0, \C)$ and its reversal. This enables posterior sampling by starting from the Gaussian measure and solving the SDE backwards in time. 
Given $T>0$, let $\{X_t\}_{t=0}^T$ stand for the infinite-dimensional diffusion process for a continuous time $t\in [0,T]$ satisfying the following SDE
\begin{eqnarray}\label{eq: diffusion model; initiated from posterior} 
    dX_t = -\frac{1}{2} X_t dt +  \C^{1/2} dB_t,
    \end{eqnarray}
with $X_0 \sim \mu^y$, where $\mu^y$ is the posterior in \eqref{eq:posterior} with prior $\mu$. The process $X_t$ corresponds to images from the posterior distribution with added noise.
The conditional distribution of $X_t|X_0$ is given by
\begin{eqnarray*}
    \L(X_t | X_0) = \N(e^{-t/2}X_0, (1-e^{-t})\C),
\end{eqnarray*}
where  $\L(\cdot)$ denotes the law of a random quantity.
The process transports $\mu^y$ to $\N(0, \C)$ as $T\rightarrow\infty$. The reverse-time dynamics are determined by the conditional score function associated with the posterior distribution, defined as follows.
\begin{definition}[\cite{baldassari2023conditional}] \label{def:conditional_score_infiniteD}
    Define the  conditional score function $s(x,t; \mu^y)$ for $x \in H$ as
\begin{equation}
\label{eq:conditional_score_infiniteD}
       s(x, t; \mu^y) = \fabian{-\frac{1}{1-e^{-t}}}\left(x - e^{\frac{-t}{2}}\E(X_0 | Y = y, X_t = x)\right).
\end{equation}
\end{definition}

Using the arguments of \cite{PidstrigachInfiniteDiffusion2024}, it can be shown that, in finite dimensions, the conditional score defined in \eqref{eq:conditional_score_infiniteD} reduces to the usual gradient of the logarithm of the marginal density. This result will later be useful in the analysis of the score mixture introduced in Section \ref{sec:reg-UCoS}.
\begin{lemma} \label{lem:score-finiteD}
    Assume that $H = \revise{ \R^n}$ and the posterior $\mu^y$ admits a density $q^y$. \lam{Then the conditional score function satisfies}
\begin{eqnarray*}
    s(x, t; \mu^y) = \C \nabla_x \log q_t^y(x) = \C \nabla_x \log \left( q^y \ast k_t\right)(x),
\end{eqnarray*}
where the convolution kernel $k_t$ is given by
\begin{eqnarray}\label{eq: convolution kernel}
    k_t(x; z) = (2 \pi (1-e^{-t}))^{\revise{-n/2}} \revise{\det(\C)^{-1/2} \exp \left(-\frac{\left\|x-e^{-t/2}z\right\|_{\C}^2}{2(1-e^{-t})}\right)}.
\end{eqnarray}
\end{lemma}
To establish the well-posedness of the reverse-time SDE underlying posterior sampling, we impose the following boundedness assumption on the score function. 
\begin{assumption}\label{ass: SBD assumption}
    Let $\C$ be a trace class covariance operator on $H$. Denote by $H_{\C}= \C^{1/2}(H)$ the Cameron-Martin space associated  with $\C$ (see \cite{schneider2025scalablediffusionposteriorsampling, da2006introduction} for details). Assume that $\mu(H_\C) = 1$ and that, for every $T> 0$, 
    \begin{equation*} 
        \sup_{t \in [0, T]} \E \norm{s(X_t, t; \mu^y)}_H^2 < \infty.
    \end{equation*}
\end{assumption}
The following proposition allows us to reverse the SDE associated with the diffusion process \eqref{eq: diffusion model; initiated from posterior}.
\begin{proposition}[\cite{baldassari2023conditional}] 
\label{thm:reverse-time-well-posed}
    Let Assumption \ref{ass: SBD assumption} hold and let $Z_t$ solve the following SDE
    \begin{equation}\label{eq: BSDE}
    \begin{aligned}
    dZ_t &= \left(\fabian{\frac{1}{2} Z_t +s(Z_t, T-t; \mu^y)}\right) dt + \C^{1/2} d B_t, \\
    Z_0 & \sim  \L(X_T).
    \end{aligned}
    \end{equation}
    Then the terminal distribution of $Z_t$ satisfies
    \begin{equation*}
        \L(Z_T) = \L(X_0) = \mu^y.
    \end{equation*} 
\end{proposition}
In a nutshell, Proposition \ref{thm:reverse-time-well-posed} shows that samples from the posterior distribution $\mu^y$ can be generated by solving the reverse-time SDE \eqref{eq: BSDE}, starting from samples drawn from the Gaussian measure $\N(0,\C)$. This principle underpins the ideas of posterior sampling using SBD models.

\subsubsection{DPS: an unconditional score approximation}
Unconditional approaches to posterior sampling with score-based diffusion models are based on the unconditional score function $s(x,t;\mu)$, which depends only on the prior distribution and can therefore be reused across multiple inverse problems without retraining. Samples from the posterior distribution are obtained by approximately incorporating the measurement information through the forward map during sampling.

In infinite dimensions, the unconditional score function associated with the prior measure $\mu$ is defined analogously to \eqref{eq:conditional_score_infiniteD} by (see \cite{PidstrigachInfiniteDiffusion2024} for the detailed derivation)
\begin{equation}
\label{eq:unconditional_score_infiniteD}
       s(x, t; \mu) = \revise{ -\frac{1}{1-e^{-t}} }\left(x - e^{\frac{-t}{2}}\E(X_0 | X_t = x)\right),
\end{equation} 
where the diffusion process is initialized with $X_0\sim\mu$.
DPS, introduced in \cite{chung2023diffusion}, approximates the conditional score function by promoting data consistency with the observed measurements. Compared to the conditional score-based approach (Section \ref{subsec:ConditionalSBD}), DPS is attractive since the learned score function is independent of the measurement data. However, repeated evaluations of the forward map are required during sampling, leading to significantly higher computational cost. Moreover, since the conditioning is only approximate, exact fidelity with the posterior distribution is not guaranteed.
In DPS, the conditional score function is approximated by 
\begin{eqnarray}\label{eq: DPS}
    s(x, t; \mu^y) \approx s(x, t; \mu) - \rho \, \revise{\C}\nabla_x \norm{y - A\left(\E(X_0| X_t = x)\right)}_{\Gamma_{\text{obs}}}^2,
\end{eqnarray}
where $\rho \in \R$ is a hyper-parameter. Here $\E(X_0| X_t = x)$ denotes the conditional expectation of the process $X_t$ in \eqref{eq: diffusion model; initiated from posterior}, initialized with $X_0 \sim \mu$. This expectation can be represented as a function of the unconditional score function $s(x,t;\mu)$, see \eqref{eq:unconditional_score_infiniteD} and \cite{chung2023diffusion}. In practice, the gradient term in \eqref{eq: DPS} can be computed using automatic differentiation in PyTorch python package \cite{NEURIPS2019_pytorch}. 
For more details for the implementation, we refer to \cite{chung2023diffusion}.


\subsubsection{UCoS: an efficient diffusion posterior sampling}
\lam{One of the drawbacks of the conditional SBD is that learning the conditional score $s(x,t;\mu^y)$ in \eqref{eq:conditional_score_infiniteD} requires the neural network to map a high-dimensional joint input space. This dramatically increases the required sample complexity (the amount of data needed to train effectively). UCoS, a more computationally favourable strategy to evaluate $s(x,t;\mu^y)$, was recently developed in \cite{schneider2025scalablediffusionposteriorsampling}. The key idea behind UCoS is to express the conditional score as an affine transformation of an unconditional score function. This yields computational advantages over the conditional approach, since the learned score does not depend on measurement data $y$, simplifying the function approximation during training. Moreover, the underlying formulation is exact and avoids repeated evaluations of the forward map during sampling, leading to significantly faster inference compared to classical unconditional approaches such as DPS in large-scale problems.}

\lam{
Since UCoS forms the backbone of the regularized version developed in the next section, we summarize the derivation in slightly more detail below. At its core, UCoS introduces a modified diffusion process $\tilde{X}_t$, whose score coincides with the conditional score up to affine transformations that depend on the forward operator. Although this transformation initially still involves the forward map, a further reparameterization removes this dependence during sampling. As a result, only a single adjoint evaluation is required, and no further forward evaluations are needed during inference.} 
Introduce the notations
\begin{eqnarray*}
   \Sigma_t &=& (e^t -1)\C - (e^t-1)^2\C A^* C_t^{-1} A \C, \qquad C_t = (e^{t}-1) A \C A^* + \Gamma_{\text{obs}} \in \R^{m \times m}.
\end{eqnarray*}
Given these objects, define the stochastic process $\tilde X$ 
\begin{equation*}
    \widetilde X_t = \widetilde X_0 + Z_t, \quad \widetilde X_0 \sim \mu \; \text{and} \; Z_t \sim \N(0, \Sigma_t)
\end{equation*}
with corresponding score function
\begin{eqnarray*}
    \tilde{s}(z, t; \mu) = - \left(z -\E(\widetilde X^{\mu}_0 |\widetilde X^{\mu}_t = z)\right).  
\end{eqnarray*}
Let the affine transformation function $r$ of the score function $\tilde s$ under $R_t = \Sigma_t \C^{-1}$ be given by
\begin{eqnarray}\label{eq: r in UCoS}
    r(\eta, t; \mu): = \tilde{s}(R_t \eta, t; \mu) + R_t \eta.
\end{eqnarray}
\lam{The following theorem, established in \cite{schneider2025scalablediffusionposteriorsampling}, provides the key representation underlying efficient score estimation.}
\begin{theorem}
    Let Assumption \ref{ass: SBD assumption} hold.
    Then the conditional score function satisfies
    \begin{eqnarray}\label{eq: UCoS}
        s(x, t; \mu^y) = \lambda(t) \left(r(\xi_t(x;y), t; \mu) - \revise{e^{t/2}} x\right),
    \end{eqnarray}
    where
    \begin{eqnarray*}
    \xi_t(x,y) &=& \C A^* \revise{\Gamma_{\rm{obs}}^{-1}} y + \lambda(t) x, \qquad
    \lambda(t) = \left(e^{t/2}-e^{-t/2} \right)^{-1}.
    \end{eqnarray*}
\end{theorem}
\subsubsection*{Training in UCoS}
In order to train an approximation $r_\theta \approx r$, define score matching (SM) loss function
\begin{equation*}
     \mathrm{SM_{UCoS}}(r_\theta) := \E\left[\lambda(t)^2\norm{r(R_t^{-1} \tilde x_t,t ; \mu) -r_\theta(R_t^{-1} \tilde x_t, t; \mu)}_H^2\right].
\end{equation*}
\lam{Direct computation of \(r\) is infeasible in practice since the true score function is unknown.} To address this, the denoising score matching (DSM) \cite{Vincet2011AConnection} is introduced:
\begin{eqnarray}\label{eq: DSM}
    \mathrm{DSM_{UCoS}}(r_\theta) := {\mathbb E} \left[\lambda(t)^2 \left\|r_\theta(R_t^{-1} \tilde x_t, t; \mu) - \tilde x_0\right\|_H^2\right].
\end{eqnarray}
By \cite[Lemma 3.10]{schneider2025scalablediffusionposteriorsampling}, $\mathrm{SM}(r_\theta ) = \mathrm{DSM}(r_\theta) + V$, where $V$ is independent of $\theta$. Therefore $\mathrm{DSM}$ can be minimized with a known target. To efficiently sample $R_t^{-1} \tilde x_t$ without computing matrices corresponding to $A$ or their inversion, \lam{one can use the following identity (see \cite{schneider2025scalablediffusionposteriorsampling})}
\begin{eqnarray}\label{eq: R_tx_t}
    R_t^{-1} \tilde x_t = \left(\frac{1}{e^t-1}I + {\mathcal C} A^* \Gamma_{\text{obs}}^{-1}A\right) \tilde x_0 + \frac{1}{\sqrt{e^t-1}}{\mathcal C}^{1/2} z_1 + {\mathcal C} A^* \Gamma_{\text{obs}}^{-1/2} z_2,
\end{eqnarray}
in distribution, where $(\tilde x_0, \tilde x_t) \sim {\mathcal L}(\widetilde{X}_0, \widetilde{X}_t)$ and $z_1, z_2 \sim N(0, I)$ are independent. In Section \ref{sec: training and eval}, we give an algorithmic description of the training procedure which approximates the score function through $r_\theta$ in UCoS. Further the posterior sampling with UCoS is described algorithmically.

\section{Regularization based on Gaussian approximation}
\label{sec:reg-UCoS}

The authors of \cite{Baptista2025MemorizationAR} introduce Tikhonov-regularized score matching, a regularized variant of score matching that stabilizes the score approximation by adding a quadratic norm penalty to the loss function:
\begin{equation}
    \label{eq:Tik_reg_score_match}
     \mathrm{RSM}(s_\theta) := \E\left[\lambda(t)^2\left(\norm{s( x_t,t ; \mu^y) -s_\theta( x_t, t; \mu^y)}_H^2) + \norm{s_\theta(x_t, t)}_{\fabian{\Gamma_t^{-1}}}^2\right)\right].
\end{equation}
Here, $\norm{\cdot}_{\Gamma_t^{-1}}\fabian{:= \langle \cdot, \Gamma_t \cdot \rangle}$ stands for the $H$-norm being weighted by \fabian{an invertible}, linear time-dependent operator $\Gamma_t : H \to H$. A non-centered version of the approach in \eqref{eq:Tik_reg_score_match} is obtained by setting
\begin{equation*}
     \revise{\widetilde{\mathrm{RSM}}}(s_\theta) := \E\left[\lambda(t)^2\left( \norm{s(x_t,t;\mu^y)-s_\theta(x_t,t;\mu^y)}_H^2 +\norm{s_\theta(x_t,t)-s_0(x_t,t)}_{\fabian{\Gamma_t^{-1}}}^2
\right)\right].
\end{equation*}
where $s_0$ is a known (potentially conditional) score function corresponding to another probability measure. The straightforward motivation of using a non-centered regularizer $s_0$ is two-fold: first, centering the regularizer with a closed-form $s_0$ such as one emerging from a Gaussian can help to encode expert information into the conditional distribution which is not necessarily well-presented in training data. Second, the approach could be used to leverage or improve a score obtained earlier through another training data set e.g. with different resolution. 

The immediate counter-argument in the context of \emph{finite training data} is that the information encoded in $s_0$ is not propagated to $s_\theta$ outside the training data set. Therefore, the estimated score may miss specific expert information such as the tail behaviour. Moreover, the choice of $\Gamma_t$ should reflect the principle that the regularization strength should diminish as the amount of training data grows.

In our quest to better enforce the expert information in scenarios with limited training data, we make the following observation.
Modifying \cite[Theorem 5.1]{Baptista2025MemorizationAR}, it is straightforward to show that in the infinite data limit the Tikhonov regularized score matching objective is minimized by 
\begin{eqnarray*}
    s_{\text{Tik},0}(x, t; \mu^y) = \left( I + \Gamma_t \right)^{-1} s(x, t; \mu^y) + \left( I + \Gamma_t \right)^{-1} \Gamma_t s_0(x,t),
\end{eqnarray*}
where $s$ is the (true) score function. For $\alpha \in (0, 1)$, the choice $\Gamma_t \equiv\frac{\alpha}{1-\alpha} I $ yields
\begin{equation}
    \label{eq:sTik0}
     s_{\text{Tik},0}(x, t; \mu^y) = (1-\alpha) s(x, t;\mu^y) + \alpha \, s_0(x,t).
\end{equation}
Motivated by this identity, we utilize a regularized score function $\revise{\hat s_{\text{Tik}, \theta}(\cdot, \cdot; \mu^y, \nu^y)}$ of the form
\begin{eqnarray}\label{eq: reg score}
    \revise{\hat s_{\text{Tik}, \theta}(x, t; \mu^y, \nu^y)}:=
     (1-\alpha) s_\theta(x, t; \mu^y) + \alpha \, s(x,t;\nu^y),
\end{eqnarray}
where $\mu^y$ is the data-driven posterior distribution of interest and $\nu^y$ a closed-form posterior distribution, where the prior has desirable properties for the problem at hand and the corresponding score also has a closed-form. Moreover, $s_\theta$ is the data-driven score function obtained by UCoS. 

Consequently, unlike \cite{Baptista2025MemorizationAR}, we do not incorporate the regularization into the training procedure, but instead apply it directly to the trained score function during sampling. This decouples training from the choice of regularization parameter $\alpha$, so that different values of $\alpha$ can be explored without repeated retraining of the score approximation.

An immediate example of a well-motivated regularizing distribution is the `model-based' Gaussian posterior $\nu^y = \N(\bar x, \Gamma_{\text{post}})$ described in Section \ref{subsec: modelbased}. In the finite-dimensional setting, the score function is readily available as
\begin{eqnarray}\label{eq: Gaussian score}
    s(x, t; \nu^y) = \C\left((1-e^{-t})\C + e^{-t}\Gamma_{\text{post}})  \right) ^{-1}(e^{-t/2}\bar x-x).
\end{eqnarray}
In fact, the setting can also be generalized to the infinite-dimensional setting, see \cite[Lemma 3.8]{schneider2025scalablediffusionposteriorsampling}. \revise{However, the geometric mixture interpretation developed below, in particular Theorem \ref{thm: mixture scores}, is restricted to finite dimensions (i.e. $H=\R^n$), as in \cite{Baptista2025MemorizationAR}. A rigorous extension to infinite dimensions would require additional measure-theoretic assumptions on the two distributions and is left for future work.}

For convenience, we assume $\C = I$ in the remainder of this section. Recall that (see Lemma \ref{lem:score-finiteD}) the score function of a generic distribution $\mu$ with PDF $p_0$ simplifies to
\begin{eqnarray*}
    s(x, t; \mu) = \nabla_x \log \left( p_0 \ast k_t\right)(x),
\end{eqnarray*}
where the convolution kernel $k_t$ is given by \eqref{eq: convolution kernel}. Therefore, it follows that, when $\nu_1$ and $\nu_2$ admit regular PDFs $p_1, p_2$ on \revise{$\R^n$},  
    \begin{eqnarray}   
        \label{eq:score_identity}
          \alpha s(x, t; \nu_1) + (1-\alpha) s(x, t; \nu_2)
        & = & \alpha \nabla \log (p_1 \ast k_t) + (1-\alpha) \nabla \log (p_2 \ast k_t) \nonumber \\
        & = & \nabla \log \left[(p_1\ast k_t)^{\alpha} (p_2 \ast k_t)^{1-\alpha} \right].
    \end{eqnarray}
It is now natural to ask whether the right-hand side of identity \eqref{eq:score_identity} approximates any score function in relevance to the diffusion process \eqref{eq: diffusion model; initiated from posterior}. We demonstrate by the next theorem that indeed as $t$ decreases, the score on the right-hand side of \eqref{eq:score_identity} approximates locally the score of a density $(q_\alpha/\int q_\alpha(x)dx) \ast k_t$, where
\begin{eqnarray}\label{eq: pdf of mixture}
    q_\alpha(x) := p_1^{\alpha}(x)p_2^{1-\alpha}(x), \quad x \in \revise{\R^n},
\end{eqnarray}
is the unnormalized geometric mixture of $p_1$ and $p_2$.
\lam{
\begin{remark}
    We observe that the modes of the geometric mixture $q_\alpha$ `interpolate' between the modes of the underlying distributions $p_1, p_2$. To illustrate this behaviour, suppose that $p_1$ is the density of a symmetric mixture of two Gaussian distributions with common variance $\sigma_1^2$ and isolated modes $m_1^-$ and $m_1^+$, while $p_2$ is a Gaussian density with variance $\sigma_2^2>\sigma_1^2$ and mode $m_2 = \tfrac{m_1^-+m_1^+}{2}$. Then, for $\alpha$ close to 1, the two modes of the geometric mixture $q_\alpha$ lie between $m_1^-$ and $m_2$, and between $m_2$ and $m_1^+$, respectively. Moreover, as $\alpha \to 1$, these modes approach $m_1^-$ and $m_1^+$, whereas as $\alpha \to 0$, they approach $m_2$; see \cite[Section 3.3]{ChopinNicolas2020Aits}. This behaviour
    contrasts with that of arithmetic mixtures $\alpha p_1 + (1-\alpha)p_2$. In the example illustrated in Figure~\ref{fig:mixtures}, the arithmetic mixture mainly re-weights the component densities while preserving all three visible modal structures. 
    In contrast, the geometric mixture shifts the two dominant modes toward one another rather than preserving a distinct central mode. See also the discussion in \cite{Pettigrew2026Geometric}.
\end{remark}
}
\begin{figure}[ht]
    \centering
    \begin{subfigure}{0.45\textwidth}
        \centering
        \includegraphics[width=\linewidth]{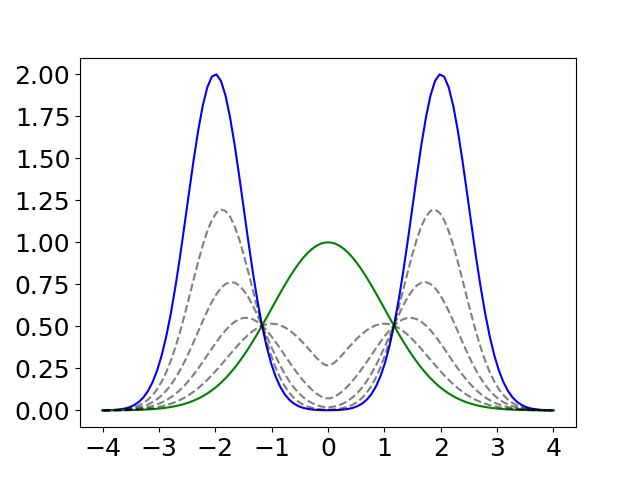}
    \end{subfigure}
    \begin{subfigure}{0.45\textwidth}
        \centering
        \includegraphics[width=\linewidth]{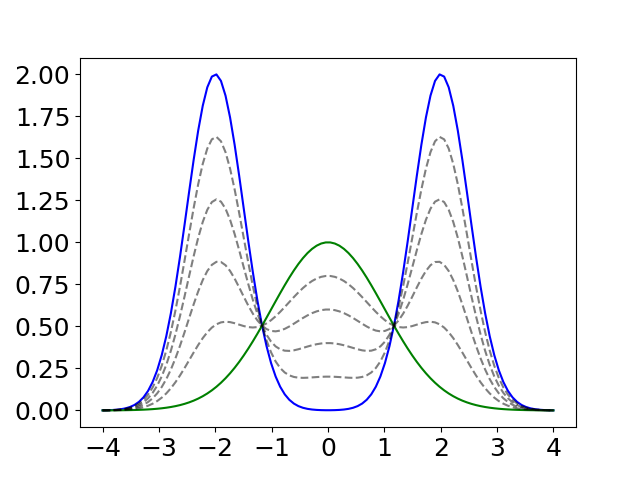}
    \end{subfigure}
    \caption{Comparison of geometric mixture (left) and arithmetic mixture (right). Blue: density of the bimodal distribution $0.5\N(-2, 0.5^2) + 0.5\N(2, 0.5^2)$ with modes $m_1^- = -2, m_1^+ = 2$, green: Gaussian density $\N(0,0.8^2)$ with mode $m_2 = 0$. Dashed lines: geometric/arithmetic mixtures corresponding to $\alpha = 0.2, 0.4, 0.6, 0.8$.}
    \label{fig:mixtures}
\end{figure}
\begin{theorem}\label{thm: mixture scores}
    Assume $p_i \in C^3_b(\revise{\R^n})$ such that $Lp_i, L q_\alpha \in L^\infty(\revise{\R^n})$, and 
    \revise{$L \nabla p_i, L \nabla q_\alpha  \in L^\infty(\R^n, \R^n)$}, $i=1,2$, where the operator $L$ is defined as
    \begin{eqnarray}\label{eq:L-operator}
            L: C^3_b(\R^d) \rightarrow C(\R^d), \quad u \mapsto \frac{1}{2} \left( \Delta u + \nabla \cdot(x u) \right).
        \end{eqnarray}
    Moreover, suppose $\Omega \subset \revise{\R^n}$ is a bounded open set and $p_1(x), p_2(x) \geq m_\Omega$ for any $x\in \Omega$.
    Then there exists a constant $C >0$ and $T>0$ such that for \revise{every $x\in \Omega$ and $0<t < T$},
    \begin{eqnarray} \label{eq:score-mixture}
        \left|\fabian{\nabla_x} \log \left[(p_1\ast k_t)^{\alpha}(x) (p_2 \ast k_t)^{1-\alpha}(x)\right]  - \fabian{\nabla_x} \log \left[(q_\alpha \ast k_t)(x)\right]\right| \le C t,
    \end{eqnarray}
    where the kernel $k_t$ is defined in \eqref{eq: convolution kernel}.
\end{theorem}

\revise{We note that the constant $C$ in \eqref{eq:score-mixture} depends on the dimension $d$, and the estimate is therefore not uniform under the refinement of the discretization. As a result, Theorem \ref{thm: mixture scores} does not establish discretization invariance of either the score mixture approximation or an empirically selected value of $\alpha$, which we leave to future work.}

\begin{remark}
    Assume the densities $p_i \in C_b^3$ are sub-exponential. More precisely, assume they satisfy $p_i(x) = \exp\left(-r_i(x)\right)$ for $i = \{1, 2\}$, where $\fabian{r_i \in C^3(\R^n)}$ \revise{with derivatives of order 2 and 3 uniformly bounded} and
    \begin{eqnarray*}
        r_i(x) \ge \revise{ c_x\norm{x} - c}, \quad \fabian{c_x> 0, c \ge 0}.
    \end{eqnarray*}
    Then the assumptions of Theorem \ref{thm: mixture scores} are satisfied.
\end{remark}
To prove Theorem \ref{thm: mixture scores}, we will need the following lemma, \lam{which provides a small-time estimate for the score of a smoothed density.}
\begin{lemma}\label{lem:rate_pt}
    Let $p$ satisfy the assumptions of Theorem \ref{thm: mixture scores}. Then there exists a constant $C>0$ and $T>0$ such that for \revise{every $x\in \Omega$ and} $0\leq t\leq T$, it holds that
    \begin{equation}
        \label{eq:pkt_converges_linearly}
        \left|\fabian{\nabla_x}\log (p \ast k_t) - \fabian{\nabla_x} \log p \right| \le C t.
    \end{equation}
\end{lemma}
\begin{proof}
    \revise{We write $\nabla$ for $\nabla_x$}. First notice that $w(t):= p \ast k_t$ satisfies the Fokker-Planck-Kolmogorov equation \cite[Section 5]{Särkkä_Solin_2019},
        \begin{eqnarray*}
            \begin{cases}
                \partial_t \, w(t) = L w(t) \\
                w(0) = p,
            \end{cases}
        \end{eqnarray*}
    where the operator $L$ is defined by \eqref{eq:L-operator}. Thus $L$ is the generator of the semigroup $S(t):  C(\revise{\R^n}) \rightarrow C(\revise{\R^n})$ defined as $p \mapsto p \ast k_t$ for $t \in [0, \infty)$ (see, e.g., \cite[Section 7.4]{EvansLawrenceC2022Pde}). We note moreover that        \begin{equation}\label{eq:contraction_S(t)}
            \|S(t)p\|_\infty \le \|p\|_\infty \int_{\R^d}k_t(x; z) \fabian{dz} = \fabian{e^{tn/2}}\|p\|_\infty.
        \end{equation}
    Now for any $p$ in the domain of $L$, thanks to the differential property of semigroups (\cite[Section 7.4, Theorem 1]{EvansLawrenceC2022Pde}) and \eqref{eq:contraction_S(t)},
        \begin{equation}\label{eq:conv_in_t_of_S}
            \|S(t)p - p\|_\infty \le \int_0^t \|L S(s) p\|_\infty ds = \int_0^t \|S(s) L p\|_\infty ds \le t\fabian{e^{tn/2}}\|L p\|_\infty.
        \end{equation}
    By applying \eqref{eq:conv_in_t_of_S} for $\nabla p$ we obtain
    \begin{equation}  \label{eq:conv_in_t_nabS}
        \|S(t)\nabla p - \nabla p\|_\infty \le t \fabian{e^{tn/2}} \|L \nabla p\|_\infty.
    \end{equation}
   \revise{Notice that $\nabla (S(t)p) = e^{t/2}S(t)\nabla p$, we have
   \begin{equation} \label{eq:grad_cov_in_t}
       \begin{aligned}
           \|\nabla(S(t)p) - \nabla p\|_\infty & \le e^{t/2}\|S(t)\nabla p - \nabla p\|_\infty + (e^{t/2} - 1)\|\nabla p\|_\infty \\
           & \le t e^{t(n+1)/2} \|L \nabla p\|_\infty + \frac{1}{2}t e^{t/2} \|\nabla p\|_\infty,
       \end{aligned} 
   \end{equation}
   thanks to \eqref{eq:conv_in_t_nabS} and the fact that $e^x - 1 \le xe^x$ for $x>0$.}
   Now \revise{write},
    \begin{equation*} 
        |\nabla \log (p\ast k_t) - \nabla \log p| \le \frac{|\nabla (p\ast k_t) - \nabla p|}{p\ast k_t} + |\nabla p|\frac{|p\ast k_t - p|}{(p\ast k_t)p}.
    \end{equation*}
    and we observe that there exists $T>0$ such that on $\Omega$ the convolution $p\ast k_t$ is bounded from below by $m_\Omega/2$. Therefore the inequality \eqref{eq:pkt_converges_linearly} follows for some $C>0$ by combining \eqref{eq:conv_in_t_of_S} \revise{and \eqref{eq:grad_cov_in_t}}.
\end{proof}

\begin{proof} [Proof of Theorem \ref{thm: mixture scores}]
For $q_\alpha = p_1^\alpha p_2^{1-\alpha}$, we have that
\begin{equation*}
    \nabla\log q_\alpha = \alpha\nabla\log p_1 + (1-\alpha)\nabla\log p_2. 
\end{equation*}
Therefore, \lam{if we denote $\mathcal{I}:= |\nabla \log \left[(p_1\ast k_t)^{\alpha} (p_2 \ast k_t)^{1-\alpha}\right] - \nabla \log (q_\alpha \ast k_t)|$, then it holds that
\begin{align*}
     \mathcal{I} & \le |\alpha \nabla \log (p_1 \ast k_t) + (1 - \alpha)\nabla \log (p_2 \ast k_t) - \nabla\log q_\alpha| + |\nabla \log (q_\alpha \ast k_t) - \nabla \log q_\alpha| \\
    & \le \alpha |\nabla\log(p_1\ast k_t) - \nabla\log p_1| + (1 - \alpha) |\nabla\log(p_2\ast k_t) - \nabla\log p_2| \\ 
    & \qquad + |\nabla \log (q_\alpha \ast k_t) - \nabla \log q_\alpha|\\
    &:= \mathcal{I}_1 + \mathcal{I}_2 +  \mathcal{I}_3.
\end{align*} 
The result follows by applying Lemma \ref{lem:rate_pt} for each of the three terms $ \mathcal{I}_1,  \mathcal{I}_2 $ and $\mathcal{I}_3$ with $p$ being $p_1, p_2$ and $q_\alpha$, respectively.}
\end{proof}

\lam{
As a final remark, the result of Theorem \ref{thm: mixture scores} provides a theoretical interpretation of the proposed regularized score function, defined in \eqref{eq: reg score}. In the small diffusion-time regime, the regularized score locally approximates the score of a geometrically interpolated distribution between the data-driven posterior and the model-based posterior. This motivates the use of the regularized score in scenarios such as when the learned score function may be sensitive to limited training data or modelling errors. In the following sections, we investigate the practical performance of the proposed method, in comparison with other methods introduced in Section \ref{sec:preliminaries}, in simulated and experimental diffuse optical tomography problems.
}

\section{Simulations}
\subsection{Data generation}\label{sec:datagen}

In the numerical studies, the domain $\Omega \subset \mathbb{R}^2$ was a circle. We considered three simulation setups as shown in Fig.~\ref{fig:2dmesh}, and listed below:
\begin{itemize}[noitemsep]
    \item \textbf{Full-view setup:} A circular domain of radius $25\:\mm$, with $20$ sources and $20$ detectors modelled as $1\:\mm$ wide surface patches located at equi-spaced angular intervals on the boundary.
    \item \textbf{Limited-view setup:} A circular domain of radius $25\:\mm$, with $10$ sources and $10$ detectors modelled as $1\:\mm$ wide surface patches placed at equally spaced angular intervals over a $180^\circ$ boundary segment.
    \item \textbf{Simulated experimental setup:} A circular domain of radius $40\:\mm$, with $16$ sources and $16$ detectors, modelled as $8\:\mm$ and $0.6\:\mm$ wide surface patches respectively, at equally spaced angular intervals on the boundary. This simulation geometry mimics a \revise{2D slice} of the experimental setup, described later in Section \ref{sec:expt}.
\end{itemize}
\begin{figure}[tb!]
\centering
\includegraphics[scale=0.65,trim={0mm 0mm 0mm 0mm},clip]{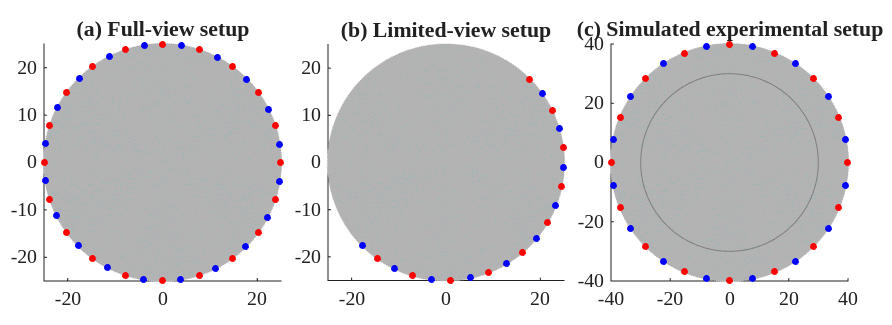}
\caption{2D simulation domains. Locations of the sources and detectors are shown as red and blue dots ($\cdot$) on the domain boundaries. Axes units are in millimetres. 
\label{fig:2dmesh}}
\end{figure}
For the full-view and limited-view configurations, the background optical parameters were set to $\mua$ = 0.01 $\mm^{-1}$ and $\mus$ = 1 $\mm^{-1}$, with an angular modulation frequency of $\omega = 100$ MHz. In the simulated experimental setup, corresponding to the time-domain DOT system (described later in Section \ref{sec:expt}), the background optical parameters were $\mua$ = 0.0065 $\mm^{-1}$ and $\mus$ = 0.95 $\mm^{-1}$ within a 30 mm radius, and $\mua$ = 0.01 $\mm^{-1}$ and $\mus$ = 0.8 $\mm^{-1}$ in the outer region. The modulation frequency was $\omega = 56.98$ MHz, corresponding to the Fourier-transformed time-domain measurements obtained from the experimental system.

The linear operators $A$ were computed for each setup using a finite element (FE) approximation of the solution to Eq. \eqref{modeldiffnorml}. The FE approximation employed the standard Galerkin method, as implemented using the Toast++ software \cite{toast}. While solving the forward problem to generate the measurements, the operator 
$A$ was computed on a 2D mesh with 12853 nodes and 25326 triangular elements. For the inverse problem, $A$ was instead computed on a mesh with 11329 nodes and 22302 triangular elements to avoid an inverse crime. 
The unknown parameters were represented as 32$\times$32 pixel images.
The random measurement noise $\epsilon$, in Eq.~\eqref{eq:ip}, was assumed to be an additive zero-mean Gaussian distribution $\N(0, \Gamma_{\mathrm{obs}})$. 
Assume further that $\Gamma_{\mathrm{obs}}$ is diagonal with standard deviation $\sigma=0.05$ for logarithm of amplitude and $\sigma=0.001$ for phase; this corresponded to roughly $1\%$ of the mean data values. While solving the inverse problem, the noise mean and covariance were assumed to be known.

\subsection{Training and evaluation}\label{sec: training and eval} 

\subsubsection*{Training}
The training dataset $(x_i)_{i = 1}^{N}$ is given by $N= 10,000$ images with 2 independent channels corresponding to absorption and scattering coefficients. Each channel has a random number of one to three circular inclusions with random position and value. The inclusion radii were drawn from a uniform distribution $U_{(0,10)}$\,mm, and their \revise{contrasts} were sampled from a uniform distribution $U_{(0,1)}$. \fabian{To augment the given training data, we rotate both channels independently by random angles $\in [0, 2\pi]$ during training.}
Training can be implemented using Algorithm~\ref{alg: training}, which is based on \eqref{eq: DSM} and \eqref{eq: R_tx_t}. 
\begin{algorithm}
\caption{Training a score approximation $r_\theta$}
\label{alg: training}
\begin{algorithmic}[1]
\State \textbf{Input:} Neural network $r_\theta$, training data $\mu =\frac{1}{N}\sum_{i=1}^N\delta_{x_i}$
\State \textbf{Output:} Trained score function $r_\theta$
\While{Training}
\State Sample $\tilde x_0 \sim \mu, t \sim U_{[\varepsilon, 1]}$ \Comment{$U_{[\varepsilon, 1]}$ is uniform distribution on $[\varepsilon, 1]$}
\State Sample $z_1, z_2 \sim \N(0, I)$
\State $R_t^{-1} \tilde x_t  \gets \left(\frac{1}{e^t-1}I + {\mathcal C} A^* \Gamma_{\text{obs}}^{-1}A\right) \tilde x_0 + \frac{1}{\sqrt{e^t-1}}{\mathcal C}^{1/2} z_1 + {\mathcal C} A^* \Gamma_{\text{obs}}^{-1/2} z_2$
\State Perform minimization step on  $  \revise{\lambda(t)^2}\left\|r_\theta(R_t^{-1} \tilde x_t, t; \mu) - \tilde x_0\right\|^2$.
\EndWhile
\end{algorithmic}
\end{algorithm}

The model is trained for 20 epochs using the AdamW \cite{Loshchilov2017DecoupledWD} optimizer, a modification of the Adam optimizer, which integrates weight decay \revise{into} the gradient. Learning rate is chosen in range from $0.002$ to $0.0005$, decaying linearly through the epochs. The truncation hyper-parameter $\varepsilon$ in \fabian{Algorithm} \ref{alg: training} is set to $\varepsilon = 0.001$. 
The operator $r_\theta$ is parametrized as a Fourier-Neural Operator \cite{li2021fourier}. 
The architecture is based on \cite{baldassari2023conditional} and consists of three main components. First, a fully connected lifting layer maps the input channels, corresponding to the unknown $x$, time $t$ and, for the conditional approach, measurement $y$, to a $32$-dimensional lifted space. This is followed by $10$ Fourier neural operator layers with $32$ nodes and $17$ Fourier modes,
and finally two fully connected layers with nonlinear activations that map back to the original resolution.
On an NVIDIA A100 GPU with 80\,GiB of memory, one epoch takes approximately 40 seconds, with an actual memory usage of less than 5\,GiB during training.
Furthermore, similar to \cite{baldassari2023conditional, schneider2025scalablediffusionposteriorsampling}, a time change $\beta(t) = 0.05 + t(10 - 0.05)$ in the forward and backward SDE is employed. \fabian{Figure~\ref{fig:data_aug} shows that, when data augmentation is applied, the validation loss decreases proportionally the training loss. The validation loss is even lower than the training loss, since the random rotations may reduce the sharpness of the inclusions, leading to more complex images. This contrasts with the non-augmented setting, where the validation loss is consistently higher than the training loss, reaching a plateau whilst the training loss is further decreasing, indicating overfitting to the training data.}
\begin{figure}[htbp]
    \centering
    \begin{subfigure}{0.3\textwidth}
        \centering
        \includegraphics[width=\linewidth]{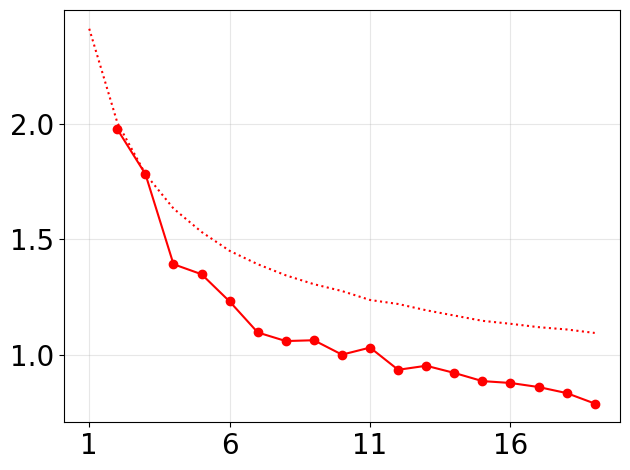}
    \end{subfigure}
    \hspace{0.03\textwidth}
    \begin{subfigure}{0.3\textwidth}
        \centering
        \includegraphics[width=\linewidth]{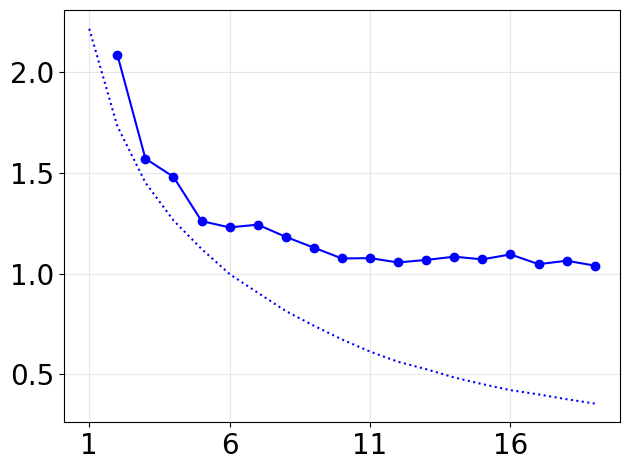}
    \end{subfigure}
    \caption{\fabian{Training loss (dotted line) and validation loss (connected line) both without data augmentation (left) and with data augmentation (right) through random rotations. X-axis: epoch, Y-axis: loss.}}
    \label{fig:data_aug}
\end{figure}

\subsubsection*{Sampling}
To sample from the posterior distribution,  the reverse SDE \eqref{eq: BSDE} under utilization of Eq.~\eqref{eq: UCoS} is discretized in time though a reverse time Euler-Maruyama method, see \cite{song2022solving, Särkkä_Solin_2019} for details. Algorithm~\ref{alg: sampling} gives more details on the precise sampling approach including the use of Euler-Maruyama method for UCoS.
\begin{algorithm}
\caption{\revise{Regularized }UCoS posterior sampling}
\label{alg: sampling}
\begin{algorithmic}[1]
\State \textbf{Input:} Score approximation $r_\theta$, measurement data $y$, time steps $T = t_k > t_{k-1} > \cdots > t_1 > t_0 = \varepsilon$, \revise{regularization strength $\alpha$}
\State \textbf{Output:} \revise{Regularized posterior sample $x_0$} 
\State $\gamma \gets \C A^* \Gamma_{\text{obs}}^{-1} y$
\State $x_{\revise{k}} \sim \N(0, \C)$
\State Precompute $\bar{x}, \revise{\Gamma_{\rm{post}}}$ given by \eqref{eq: gaussian posterior} 
\State $\Delta t \gets \revise{\frac{T-\varepsilon}{k}}, $
\For{$\revise{i = k, \cdots, 1}$} 
    \State $z_i \sim \N(0, I)$
    \State $\xi_i \gets \gamma + \lambda(t_i) \revise{x_{i}}$
    \State $s(x_i, \revise{t_i}; \mu^y) \gets \lambda(t_i) \left( r_\theta(\xi_i, \revise{t_i}; \mu) -e^{\revise{t_i} /2 } \revise{x_{i}}\right)$
    \State $s(x_i, \revise{t_i}; \nu^y) \gets  \C\left((1-e^{-\revise{t_i}})\C + e^{-\revise{t_i}}\revise{\Gamma_{\rm post}})  \right) ^{-1}(e^{-\revise{t_i}/2}\bar{x}-\revise{x_{i}})$
    \State $s_i \gets (1-\alpha) s(x_i, \revise{t_i}; \mu^y) + \alpha s(x_i, \revise{t_i}; \nu^y)$ 
    \State $\revise{x_{i-1}} \gets \fabian{x_{i}+\left(\tfrac{x_{i}}{2} + s_i\right)} \Delta t + \C^{1/2} (\sqrt{\Delta t}) z_i$
\EndFor
\end{algorithmic}
\end{algorithm}
For sampling, set $\varepsilon = 0.005$, $T = 1$ and use $k = 500$ time-steps in the reverse SDE. \fabian{For DPS, the hyper-parameter $\rho$ is selected via grid search over $10$ values, spanning from $\rho = 0.1$ to $\rho = 500$ in a roughly logarithmic scale on the limited-view problem with non-circular inclusions. The chosen value $\rho =100$ optimizes both PSNR and SSIM simultaneously and is then utilized across all experiments}.
On an NVIDIA A100 GPU, generating $1{,}000$ samples with and without regularization in parallel amounts to approximately $46$ and $41$ seconds, respectively, compared to $44$ seconds for DPS. For the numerical demonstrations, we set $\alpha = 0.5$ for regularized UCoS, based on the analysis in Section~\ref{sec: influence of alpha}.

\subsection{Results}
In this section, we generate posterior samples based on simulated measurements for the linearized DOT problem.
We compare the performance of UCoS, regularized UCoS, DPS \cite{chung2023diffusion}, and a posterior derived from a Gaussian Ornstein–Uhlenbeck process prior as outlined in Section \ref{subsec: modelbased}. The targets used for these experiments are two circular inclusions and an out-of-distribution target of two non-circular inclusions. Figure~\ref{fig:inclusion 1 absorption_scattering} corresponds to the full-view setting and Figure~\ref{fig:inclusion 2 absorption_scattering} to the limited view setup. In Table~\ref{tab: sim DOT}, we additionally report the mean squared error (MSE), peak signal-to-noise ratio (PSNR), structural similarity index (SSIM), and coverage defined as the proportion of true values falling within $1.96$ empirical standard 
deviations of the posterior mean.

We first examine the full-view results in Figure~\ref{fig:inclusion 1 absorption_scattering}. For both circular and non-circular inclusions, Gaussian posterior samples recover the scattering coefficient with moderate noise, while individual absorption samples are strongly corrupted and only capture the inclusions in terms of their mean. DPS captures the inclusions well on average but frequently introduces spurious or missing inclusions across samples. In contrast, both UCoS and regularized UCoS accurately recover the location and amplitude of the inclusions, with low variance concentrated around the inclusion regions. The visual differences between the two UCoS variants are minor.

We next consider the limited-view setting in Figure~\ref{fig:inclusion 2 absorption_scattering}, where the inverse problem is even more ill-posed. Here, the Gaussian posterior exhibits very high variance, and individual samples are heavily noise-dominated. The empirical mean recovers the scattering coefficient, while the absorption structure is largely absent. DPS recovers the scattering inclusions in the mean but fails to consistently identify both absorption inclusions, with most samples showing an incorrect number of inclusions.
UCoS and regularized UCoS provide more accurate reconstructions of the scattering inclusions, with uncertainty again concentrated near the inclusion regions. For absorption, the unregularized UCoS fails to recover the second inclusion in the circular case, which is captured by the regularized variant; in the elliptic case it is still missed. In addition, unregularized UCoS introduces a spurious third inclusion in the elliptic setting, an effect that is noticeably reduced by regularization.

In Table~\ref{tab: sim DOT}, we observe that UCoS and regularized UCoS outperform DPS and the Gaussian prior in terms of MSE, PSNR, and SSIM. Between the two UCoS variants, the regularized version achieves lower MSE and higher PSNR, while UCoS achieves slightly higher SSIM in most cases.
In terms of coverage, UCoS generally exhibits lower values than both DPS and the Gaussian prior, particularly in the limited-view setting, reflecting more concentrated uncertainty. In the limited-view case, the introduction of regularization increases coverage.
\fabian{To further verify that the SBD model is generalizing rather than memorizing the training data, we find the closest training data point to the ground truth images in Figures~\ref{fig:inclusion 1 absorption_scattering} and~\ref{fig:inclusion 2 absorption_scattering}, measured by SSIM. The corresponding minimal SSIM are given by $0.3043$ for the circular and $0.2794$ for the non-circular inclusions, whereas SBD-based posterior samples achieve a SSIM of $0.8902$ and $0.8591$ in the full-view and $0.7508$ and $0.7954$ in the limited-view setting, confirming that the SBD-based samples offer a much better reconstruction than would be possible by memorizing the data set.}

\begin{figure}[h]
  \centering
  \begin{subfigure}[b]{\textwidth}
    \centering
    \begin{subfigure}[b]{0.09\textwidth}
      \centering
      \includegraphics[width=\linewidth]{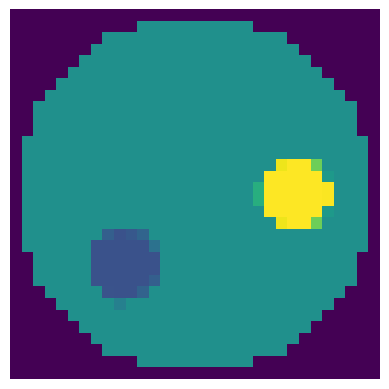}
      \vspace{1cm}
    \end{subfigure}
  \begin{subfigure}[t]{0.9\textwidth}
    \begin{subfigure}[b]{0.62\linewidth}
      \includegraphics[height=0.28\globaltextwidth]{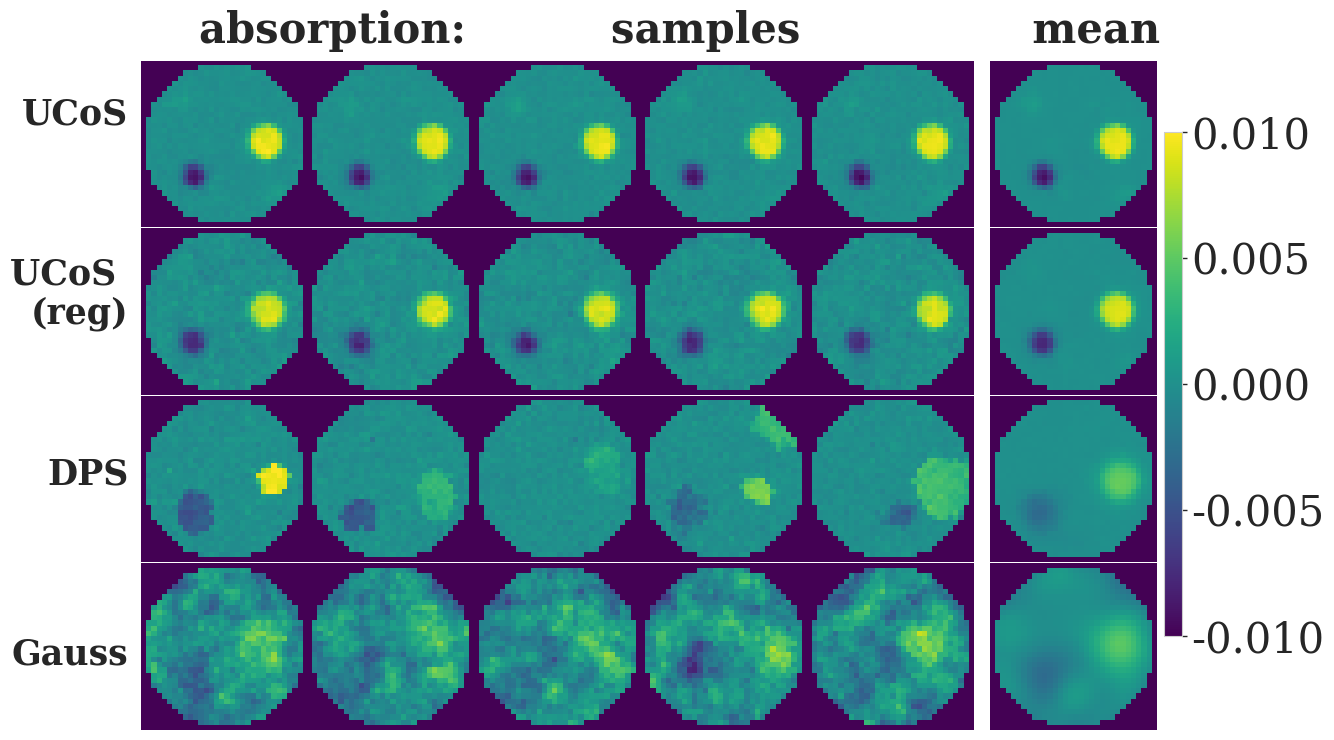}
    \end{subfigure}%
    \begin{subfigure}[b]{0.16\linewidth}
      \includegraphics[height=0.28\globaltextwidth]{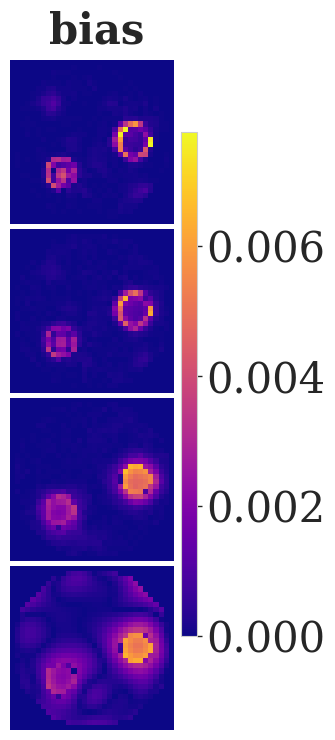}
    \end{subfigure}%
    \begin{subfigure}[b]{0.17\linewidth}
      \includegraphics[height=0.28\globaltextwidth]{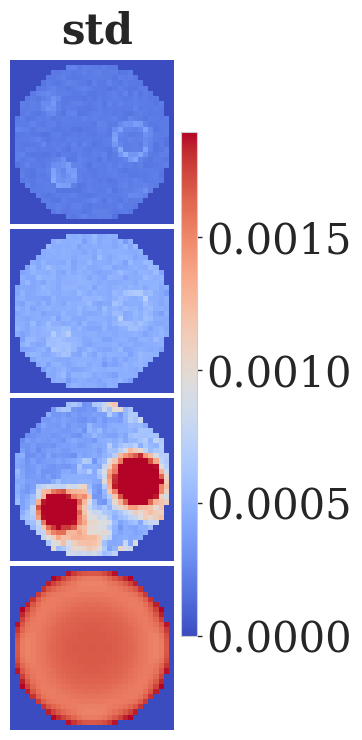}
    \end{subfigure}%
  \end{subfigure}
  \end{subfigure}
\begin{subfigure}[b]{\textwidth}
    \centering
    \begin{subfigure}[b]{0.09\textwidth}
      \centering
      \includegraphics[width=\linewidth]{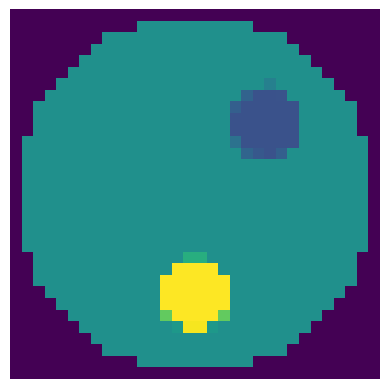}
          \vspace{1cm}
    \end{subfigure}
  \begin{subfigure}[t]{0.9\textwidth}
    \begin{subfigure}[b]{0.62\linewidth}
      \includegraphics[height=0.28\globaltextwidth]{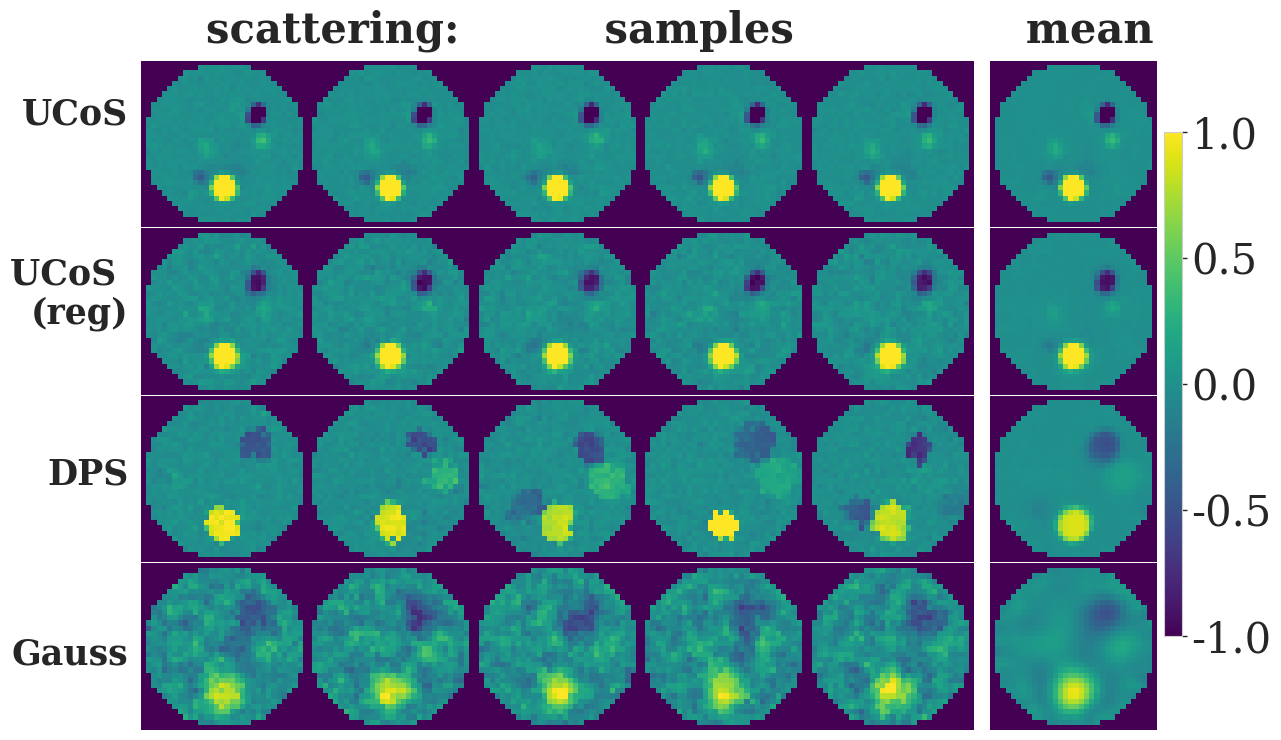}
    \end{subfigure}%
    \begin{subfigure}[b]{0.16\linewidth}
      \includegraphics[height=0.28\globaltextwidth]{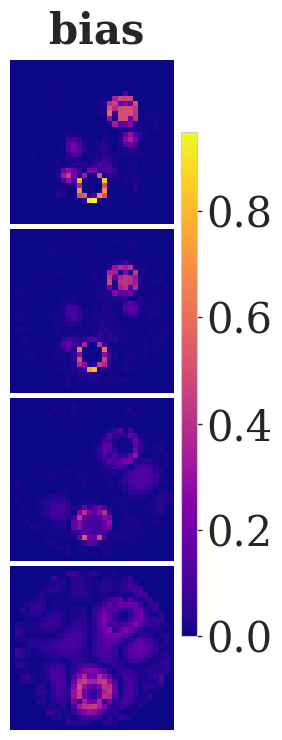}
    \end{subfigure}%
    \begin{subfigure}[b]{0.17\linewidth}
      \includegraphics[height=0.28\globaltextwidth]{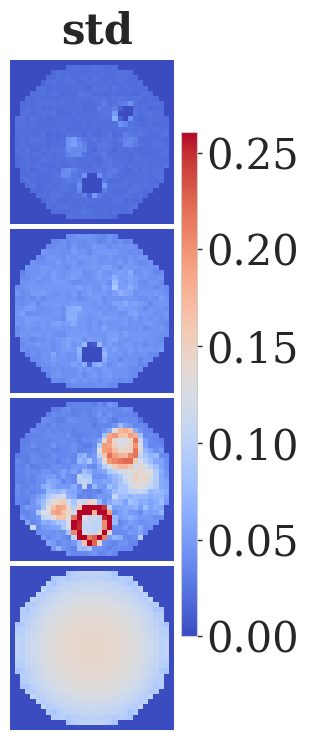}
    \end{subfigure}%
  \end{subfigure}
   \caption*{Circular inclusions}
  \end{subfigure}
    \begin{subfigure}[b]{\textwidth}
    \centering
    \begin{subfigure}[b]{0.09\textwidth}
      \centering
      \includegraphics[width=\linewidth]{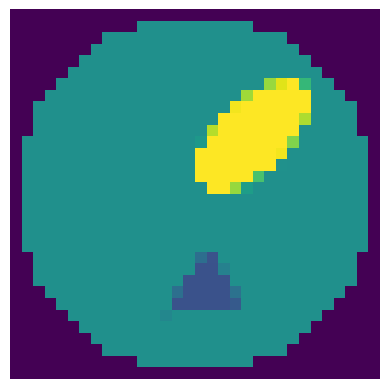}
      \vspace{1cm}
    \end{subfigure}
  \begin{subfigure}[t]{0.9\textwidth}
    \begin{subfigure}[b]{0.62\linewidth}
      \includegraphics[height=0.28\globaltextwidth]{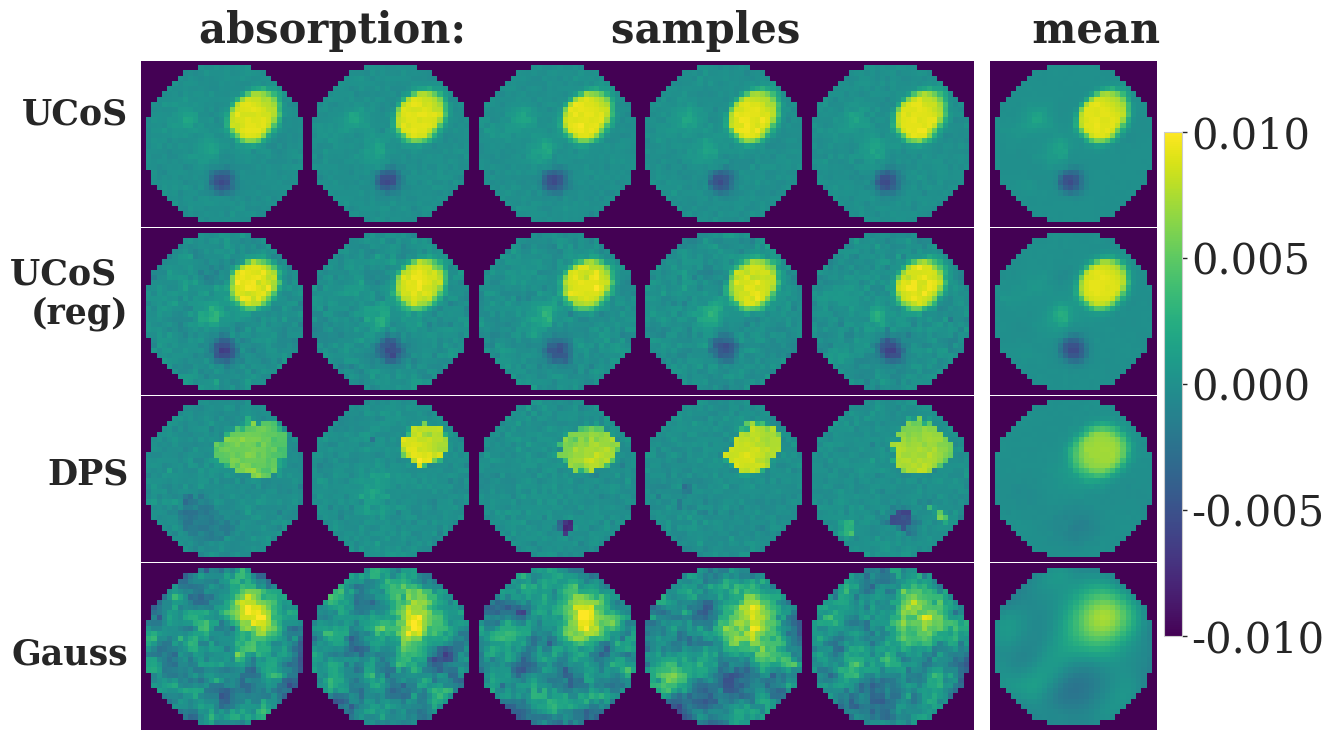}
    \end{subfigure}%
    \begin{subfigure}[b]{0.16\linewidth}
      \includegraphics[height=0.28\globaltextwidth]{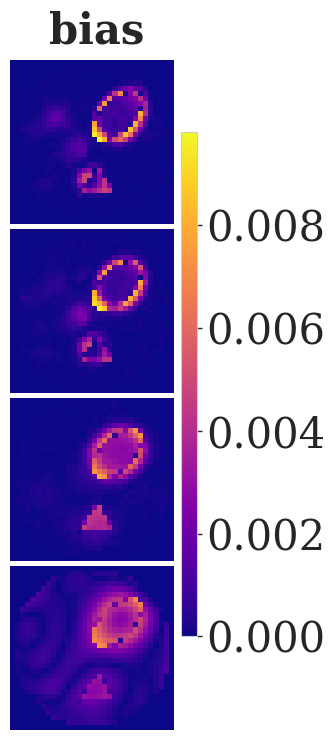}
    \end{subfigure}%
    \begin{subfigure}[b]{0.17\linewidth}
      \includegraphics[height=0.28\globaltextwidth]{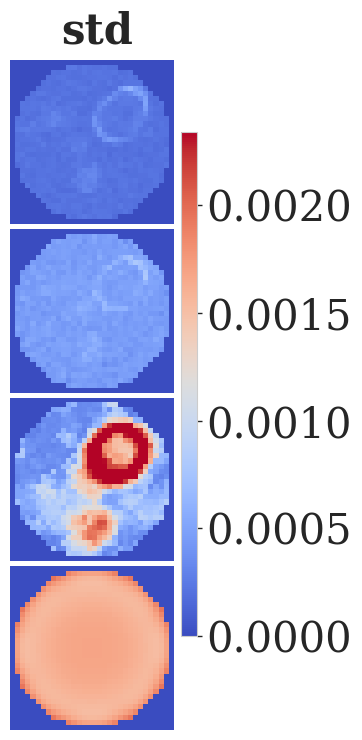}
    \end{subfigure}%
  \end{subfigure}
  \end{subfigure}
\begin{subfigure}[b]{\textwidth}
    \centering
    \begin{subfigure}[b]{0.09\textwidth}
      \centering
      \includegraphics[width=\linewidth]{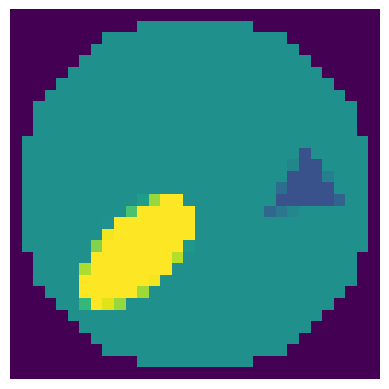}
          \vspace{1cm}
    \end{subfigure}
  \begin{subfigure}[t]{0.9\textwidth}
    \begin{subfigure}[b]{0.62\linewidth}
      \includegraphics[height=0.28\globaltextwidth]{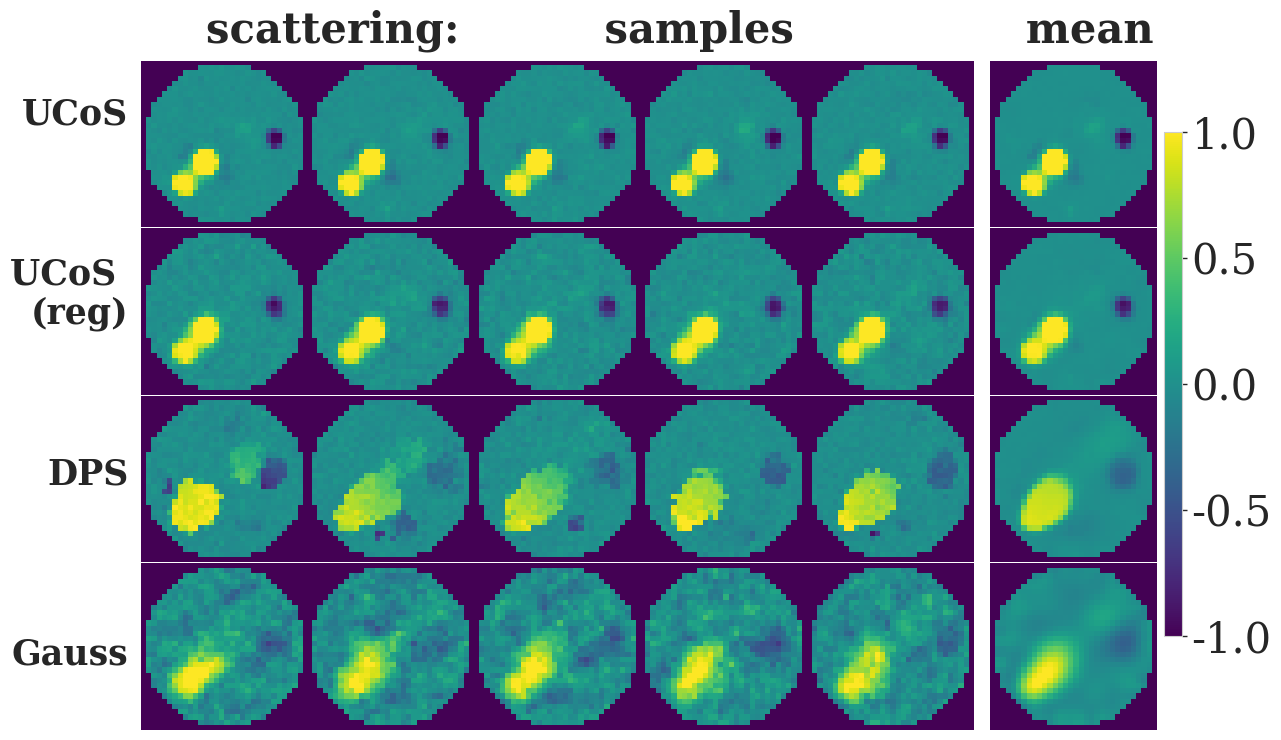}
    \end{subfigure}%
    \begin{subfigure}[b]{0.16\linewidth}
      \includegraphics[height=0.28\globaltextwidth]{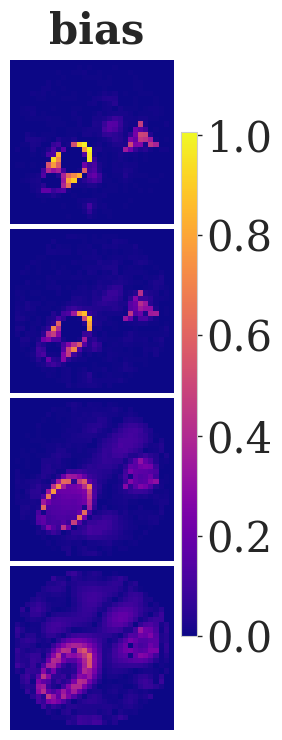}
    \end{subfigure}%
    \begin{subfigure}[b]{0.17\linewidth}
      \includegraphics[height=0.28\globaltextwidth]{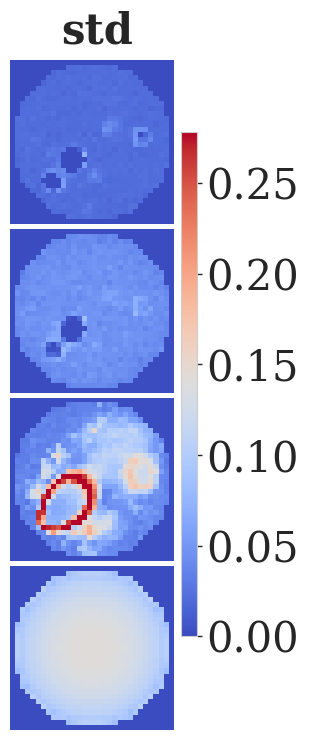}
    \end{subfigure}%
  \end{subfigure}
    \caption*{Non-circular inclusions}
  \end{subfigure}
    \caption{Simulated DOT for the full view setup. Columns from left to right: simulated ground truth (first column), five posterior samples (columns 2-6), an average (column 7), bias (column 8) and standard deviation (column 9), which are computed over 1000 generated samples. The methods used from top to bottom: UCoS, regularized UCoS, DPS, Gaussian OU process prior.}
  \label{fig:inclusion 1 absorption_scattering}
\end{figure}

\begin{figure}[h]
  \centering
\begin{subfigure}[b]{\textwidth}
    \centering
    \begin{subfigure}{0.09\textwidth}
      \centering
      \includegraphics[width=\linewidth]{figures/samples_sim/DOT/1true_absorption.png}
      \vspace{1cm}
    \end{subfigure}
    \begin{subfigure}{0.9\textwidth}
    \begin{subfigure}[b]{0.62\linewidth}
      \includegraphics[height=0.28\globaltextwidth]{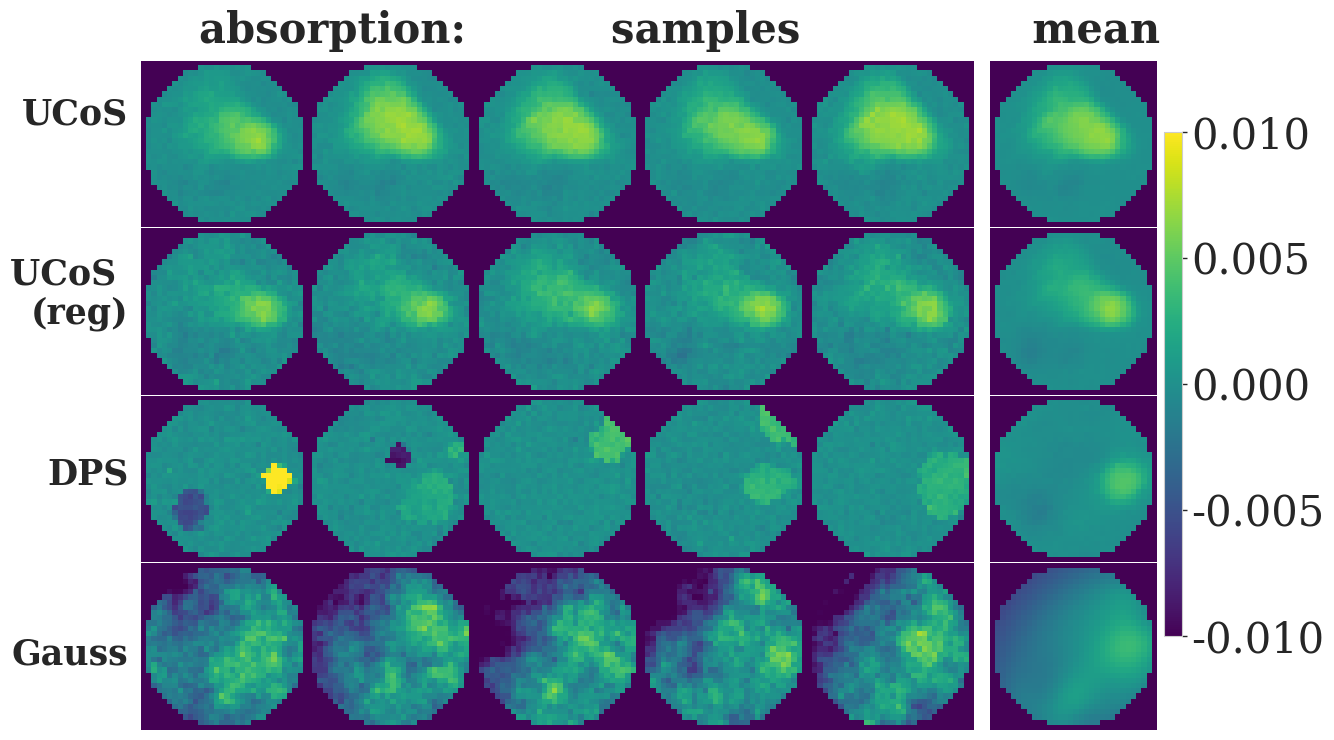}
    \end{subfigure}%
    \begin{subfigure}[b]{0.16\linewidth}
      \includegraphics[height=0.28\globaltextwidth]{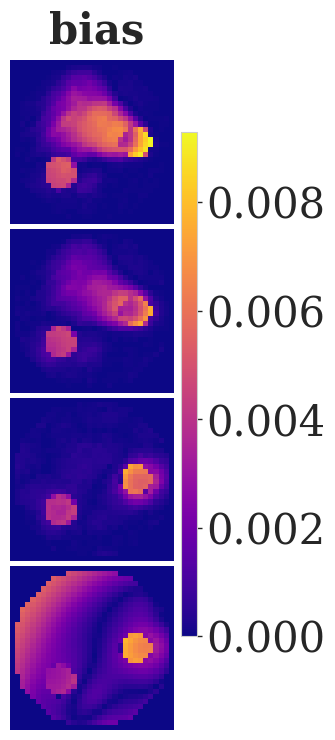}
    \end{subfigure}%
    \begin{subfigure}[b]{0.17\linewidth}
      \includegraphics[height=0.28\globaltextwidth]{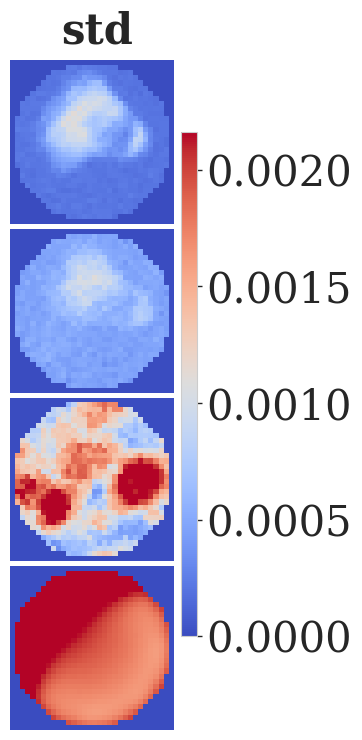}
    \end{subfigure}%
  \end{subfigure}
  \end{subfigure}
\begin{subfigure}[b]{\textwidth}
    \centering
    \begin{subfigure}[b]{0.09\textwidth}
      \centering
      \includegraphics[width=\linewidth]{figures/samples_sim/DOT/1true_scattering.png}
    \vspace{1cm}
    \end{subfigure}
    \begin{subfigure}[t]{0.9\textwidth}
    \begin{subfigure}[b]{0.62\linewidth}
      \includegraphics[height=0.28\globaltextwidth]{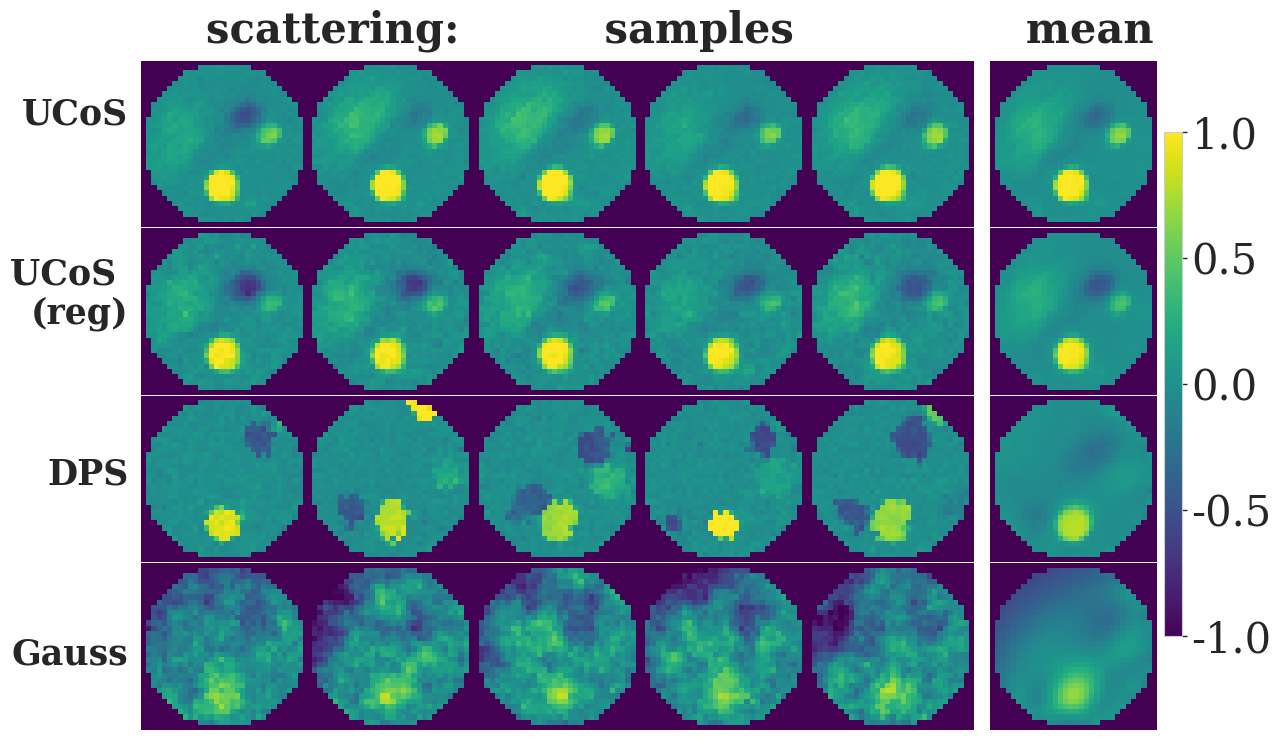}
    \end{subfigure}%
    \begin{subfigure}[b]{0.16\linewidth}
      \includegraphics[height=0.28\globaltextwidth]{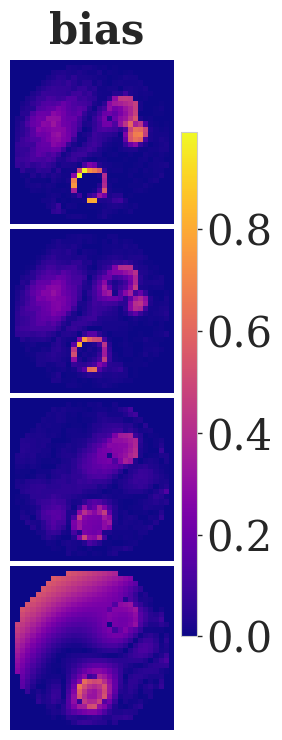}
    \end{subfigure}%
    \begin{subfigure}[b]{0.17\linewidth}
      \includegraphics[height=0.28\globaltextwidth]{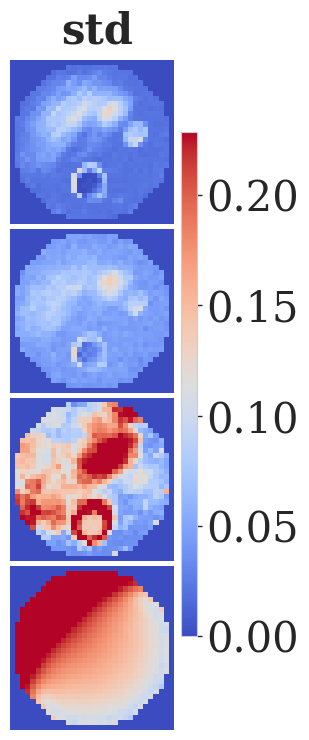}
    \end{subfigure}%
  \end{subfigure}
  \caption*{Circular inclusions}
  \end{subfigure}
 \begin{subfigure}[b]{\textwidth}
    \centering
    \begin{subfigure}{0.09\textwidth}
      \centering
      \includegraphics[width=\linewidth]{figures/samples_sim/DOT/2true_absorption.png}
      \vspace{1cm}
    \end{subfigure}
    \begin{subfigure}{0.9\textwidth}
    \begin{subfigure}[b]{0.62\linewidth}
      \includegraphics[height=0.28\globaltextwidth]{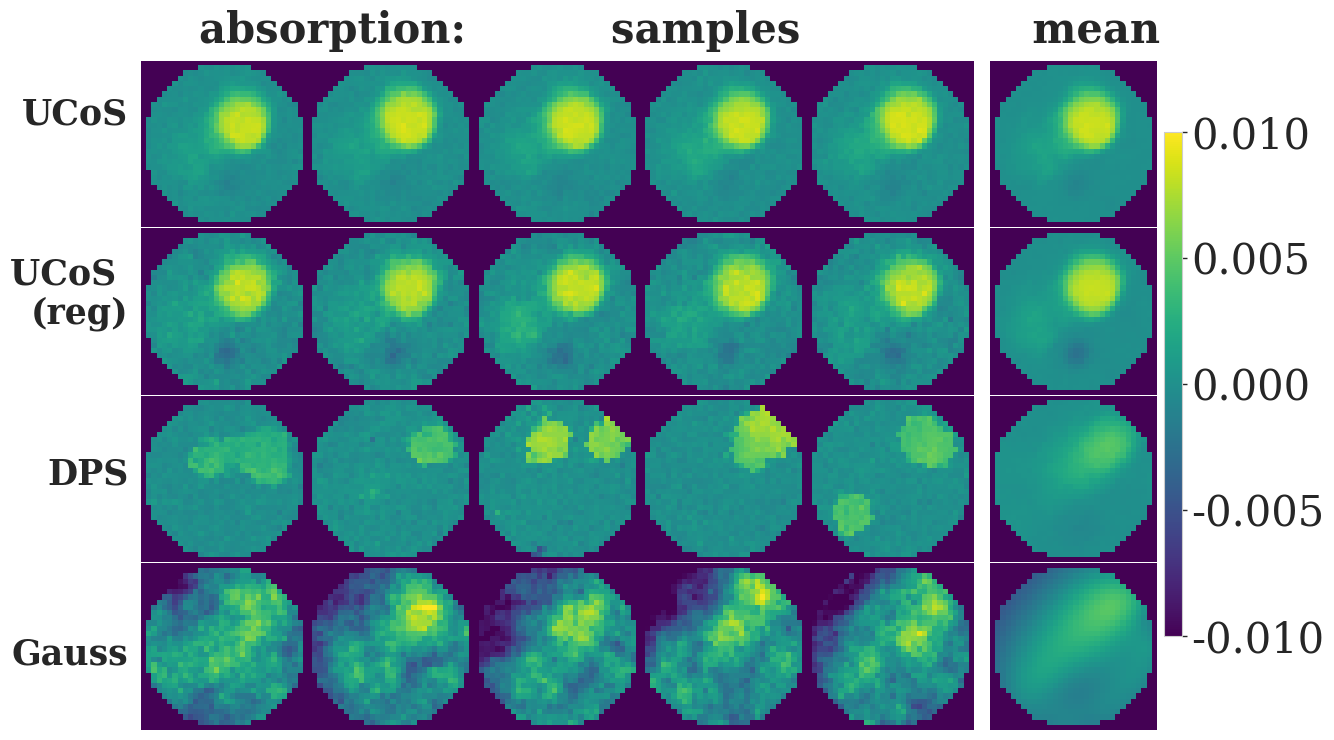}
    \end{subfigure}%
    \begin{subfigure}[b]{0.16\linewidth}
      \includegraphics[height=0.28\globaltextwidth]{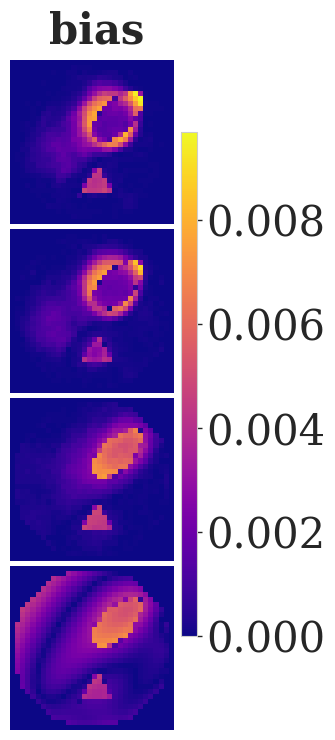}
    \end{subfigure}%
    \begin{subfigure}[b]{0.17\linewidth}
      \includegraphics[height=0.28\globaltextwidth]{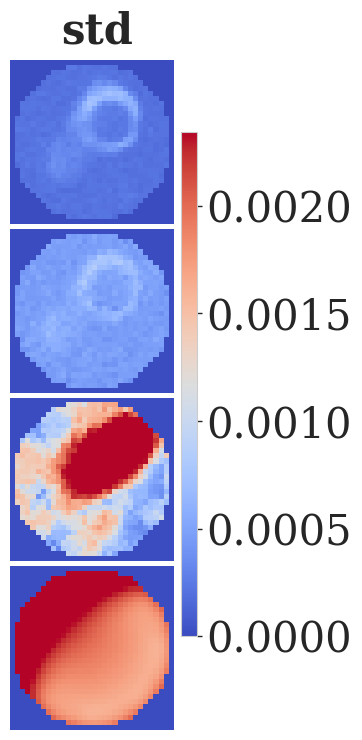}
    \end{subfigure}%
  \end{subfigure}
  \end{subfigure}
\begin{subfigure}[b]{\textwidth}
    \centering
    \begin{subfigure}[b]{0.09\textwidth}
      \centering
      \includegraphics[width=\linewidth]{figures/samples_sim/DOT/2true_scattering.png}
    \vspace{1cm}
    \end{subfigure}
    \begin{subfigure}[t]{0.9\textwidth}
    \begin{subfigure}[b]{0.62\linewidth}
      \includegraphics[height=0.28\globaltextwidth]{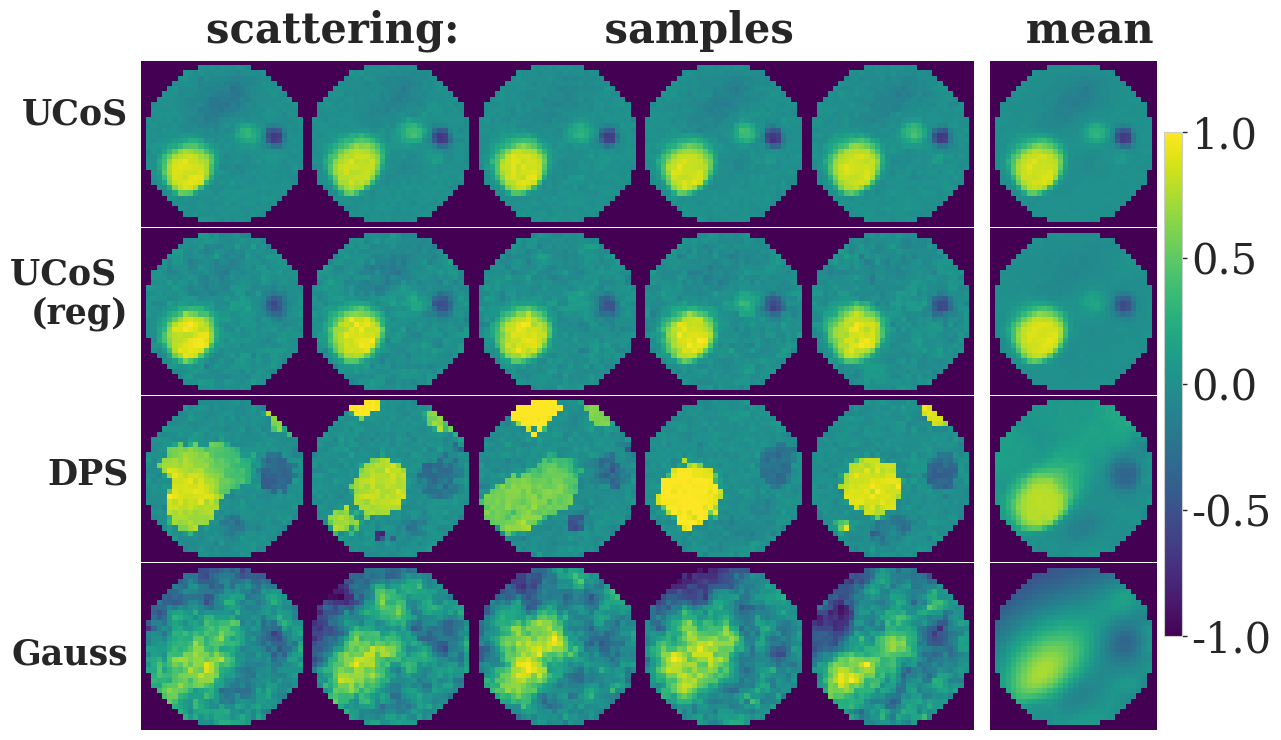}
    \end{subfigure}%
    \begin{subfigure}[b]{0.16\linewidth}
      \includegraphics[height=0.28\globaltextwidth]{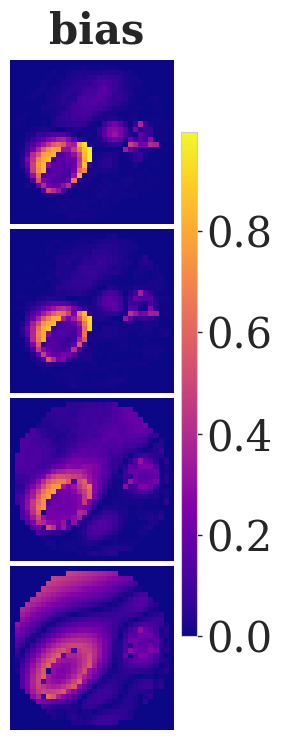}
    \end{subfigure}%
    \begin{subfigure}[b]{0.17\linewidth}
      \includegraphics[height=0.28\globaltextwidth]{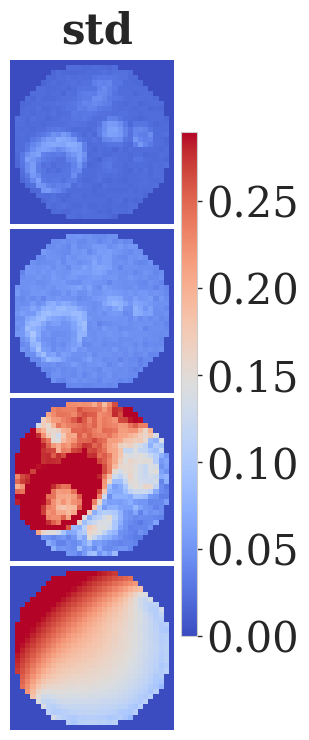}
    \end{subfigure}%
  \end{subfigure}
  \caption*{Non-circular inclusions}
  \end{subfigure}
    \caption{Simulated DOT for the limited view setup. Columns from left to right: simulated ground truth (first column), five posterior samples (columns 2-6), an average (column 7), bias (column 8) and standard deviation (column 9), which are computed over 1000 generated samples. The methods used from top to bottom: UCoS, regularized UCoS, DPS, Gaussian OU process prior.}
  \label{fig:inclusion 2 absorption_scattering}
\end{figure}

\begin{table}[]
    \centering
\begin{tabular}{c ccccc}
\hline
 \textbf{\, Full view  \, } & Metric & UCoS & UCoS (reg) & DPS &  Gaussian prior \\
\hline
\multirow{4}{*}{\shortstack{circular\\inclusions}}
& MSE & 0.0026 & \textbf{0.0020} & 0.0041 & 0.0069 \\
&PSNR & 25.7957 & \textbf{27.0083} & 24.0490 & 21.5962 \\
&SSIM & \textbf{0.8902} & 0.8690 & 0.8310 & 0.6749 \\
&Coverage & 0.9053 & \textbf{0.9365} & 0.9863 & 0.9644 \\
\hline
\multirow{4}{*}{\shortstack{non-\\circular\\inclusions}}
&MSE & 0.0043 & \textbf{0.0036} & 0.0064 & 0.0082 \\
&PSNR & 23.6944 & \textbf{24.4204} & 21.9908 & 20.8678 \\
&SSIM & \textbf{0.8591} & 0.8436 & 0.7890 & 0.6853 \\
&Coverage & 0.8677 & 0.9160 & 0.9717 & \textbf{0.9292} \\
\hline
\end{tabular}

\begin{tabular}{c ccccc}
 \textbf{Limited view} & Metric & UCoS & UCoS (reg) & DPS &  Gaussian prior \\
\hline
\multirow{4}{*}{\shortstack{circular\\inclusions}}
&MSE &  0.0094 & \textbf{0.0052} & 0.0072 & 0.0190 \\
&PSNR & 20.3505 & \textbf{22.8782} & 21.6447 & 17.2935 \\
&SSIM & 0.7508 & \textbf{0.7752} & 0.7646 & 0.5573 \\
&Coverage & 0.7383 & 0.8350 & 0.9795 & \textbf{0.9517} \\
\hline
\multirow{4}{*}{\shortstack{non-\\circular\\inclusions}}
&MSE & 0.0077 & \textbf{0.0065} & 0.0147 & 0.0176 \\
&PSNR & 21.1367 & \textbf{21.8703} & 18.5251 & 17.5957 \\
&SSIM & \textbf{0.7954} & 0.7941 & 0.7148 & 0.5916 \\
&Coverage & 0.7407 & 0.8228 & \textbf{0.9717} & 0.9204 \\
\hline
\end{tabular}
    \caption{Performance metrics for the Simulated DOT problem}
    \label{tab: sim DOT}
\end{table}

\FloatBarrier

\subsection{Influence of the regularization strength}\label{sec: influence of alpha}
\fabian{The regularisation strength $\alpha$ is a \revise{trade-off} between the learned (data-dependent) and a model-based Gaussian score function. As such, it is best tuned based on a held-out dataset drawn from a potentially perturbed but still plausible distribution. In our case this data set is given by the single ground truth with non-circular inclusions. We propose to choose a metric of interest (e.g. PSNR, SSIM or a mixture) and then choose $\alpha$ to maximise the reconstruction accuracy in that metric.}

\fabian{We now investigate the influence of the regularization parameter $\alpha$ in the regularized UCoS framework, evaluated in terms of Bias, PSNR, and SSIM, focusing on the limited-view case with two non-circular inclusions.}

\fabian{As shown in Figure~\ref{fig: alpha}, the optimal value of $\alpha$ depends on the chosen metric: the bias is minimized at $\alpha \approx 0.55$, PSNR are best at $\alpha \approx 0.6$, and SSIM is maximized at $\alpha \approx 0.35$. Overall, the variation across different values of $\alpha$ remains moderate. Since no single value optimizes all metrics simultaneously, we transform the values for PSNR, SSIM and Bias to the interval $[0, 1]$ with $1$ being the best value and then chose $\alpha =0.5$ which optimizes the sum of the three re-weighted metrics.}

\begin{figure}[h]
    \centering
    \includegraphics[width=0.95\linewidth]{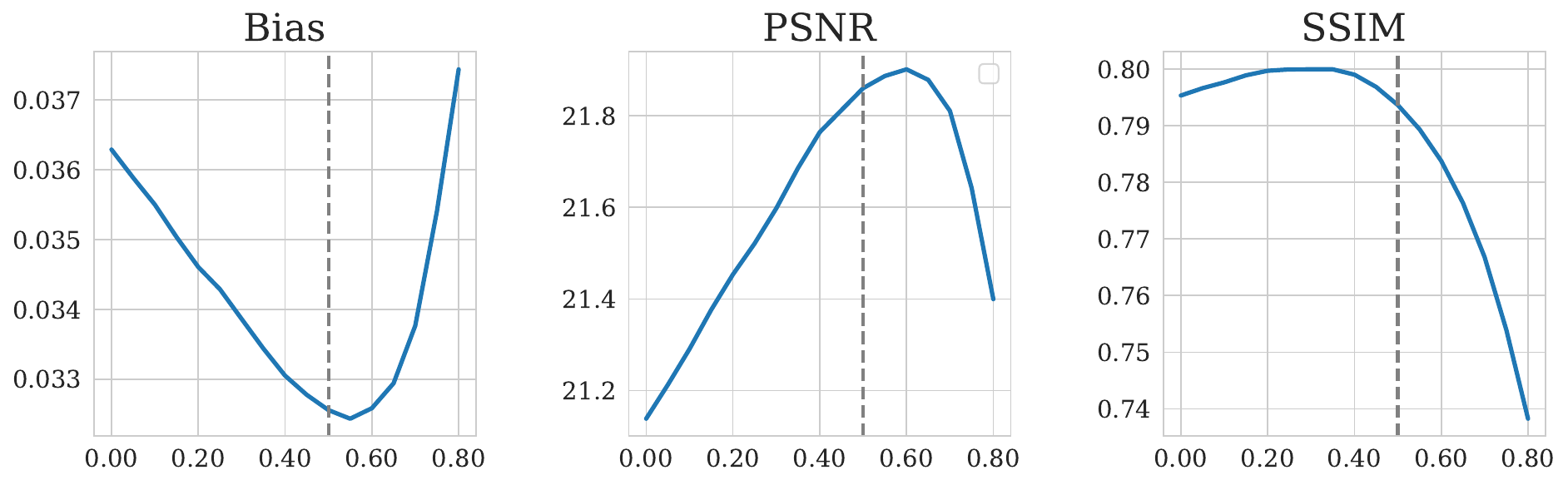}
    \caption{Evaluation of Bias, PSNR and SSIM as a function of $\alpha$. The horizontal axis shows $\alpha$, while the vertical axis reports the metric values. A dotted vertical line for $\alpha = 0.5$ is included in all plots. Metrics are computed over $1,000$ Monte Carlo samples.}
    \label{fig: alpha}
\end{figure}

\FloatBarrier


\section{Experimental}\label{sec:expt}

\subsection{Measurement setup}
The measurements were performed using the time-domain DOT setup in the University of Eastern Finland \cite{Mozumder2023setup}.
In the experiments, light pulses with an approximate duration of 200 $\sim3$\,ns were generated for each source-detector channel pair (repetition rate of 10\,Hz, pulse energy 0.12\,mJ, wavelength 808\,nm) via a Nd:YAG laser with optical parametric oscillator (NT352B, Ekspla Uab), guided to a phantom via fiber optics and collected via an avalanche photodiode (APD430A, Thorlabs).
The laser pulse-to-pulse variation was measured with the second Si biased photodetector (DET025AFC, Thorlabs).
The signals from both photodetectors were digitalized by high-bandwidth oscilloscope (WavePro 254HD, Teledyne LeCroy) with 12-bit dynamic range and sampling period of 50 ps.
The digitalized signals were then converted to frequency-domain and used for image reconstruction.
The source and detectors were located at $z=-5$\,mm and $z=5$\,mm from the vertical centre of the phantom at $z=0$\,mm.

The cylindrical phantom consisted of an outer solid and inner liquid parts (Fig.~\ref{fig:setup}).
The solid part was made of epoxy resin combined with TiO$_2$ powder and toner\cite{Zhao2022Phantom} providing absorption $\mua=0.01$\,mm$^{-1}$ and reduced scattering $\mu_s'=0.8$\,mm$^{-1}$ coefficients at 800\,nm.
The solid part had inner and outer diameters of 60\,mm and 80\,mm, respectively.
The liquid part consisted of 1\,\% SMOFlipid (Fresenius Kabi) and 0.001\,\% India ink which gave approximately $\mua=0.0065$\,mm$^{-1}$ and $\mu_s'=0.95$\,mm$^{-1}$.
Inclusions were made by filling polypropylene tubes (diameter 6\,mm, wall thickness 0.5\,mm) with different concentrations of Indian ink, intralipid or contrast agents such as indocyanine green (ICG) or light absorbing black porous silicon nanoparticles (Si NPs)\cite{Xu2018BPSi}.

Training and evaluation of UCoS for the experimental data was carried out similarly as described earlier in Section \ref{sec: training and eval}. For training, simulated data was used, as described earlier in Section \ref{sec:datagen}. During sampling, the measured and Fourier transformed data from the different targets were used as the $y$ data.

\begin{figure}[!htb]
    \centering
    \includegraphics[width=0.45\textwidth]{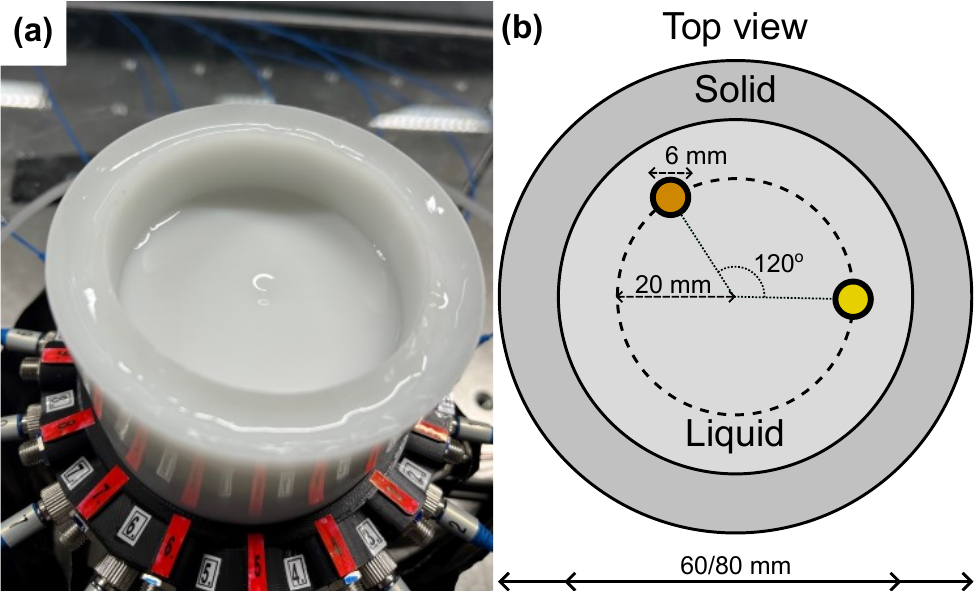}
    \caption{
    (a) Photograph of the cylindrical phantom used in DOT experiments.
    (b) Top view of the phantom with inclusions.
    }
    \label{fig:setup}
\end{figure}

\subsection{Results}
In this section, we present the experimental results shown in Figures~\ref{fig:expt1} and~\ref{fig:expt2}, where the considered targets are described in the respective captions.
The exact absorption and scattering coefficients are unknown, though they can be estimated with reasonable accuracy, and the inclusion locations are known. As the true parameters are not available, exact performance metrics cannot be computed.

In most cases, UCoS and its regularized variant produce reconstructions that closely align with the true inclusion locations, with uncertainty localized around them. Spurious inclusions occasionally appear: in Figure~\ref{fig:expt1} (a), a false inclusion is generated at the bottom of the absorption images, which disappears at higher inclusion concentrations as shown in Figure~\ref{fig:expt1} (b). In the setting where the physical scattering coefficient of the inclusion is very small, the posterior samples correctly estimate negligible scattering in Figure~\ref{fig:expt2} (b), while in Figure~\ref{fig:expt2} (a) additional 
inclusions appear at the absorption inclusion locations, which may reflect either residual scattering in the liquid or reconstruction artefacts.

In contrast, DPS exhibits substantial variability in location, size and number of the inclusions and occasionally produces artifacts that deviate from the underlying structure (e.g., the samples four and five in Figure \ref{fig:expt1}(a)).  
The sample-level consistency of DPS is generally very low, and significantly lower than that of UCoS, with the empirical average failing to capture the true inclusions, or doing so only faintly, across all experiments.

Gaussian posterior samples show large variance across the entire domain. As with DPS, the posterior mean can occasionally faintly approximate the true coefficients (e.g., Figure \ref{fig:expt1}(b) and absorption in Figure \ref{fig:expt2}), with low contrast.


\begin{figure}[h]
    \centering
    \begin{subfigure}{0.8\textwidth}
        \begin{subfigure}{0.1\textwidth}
            \includegraphics[width=\linewidth]{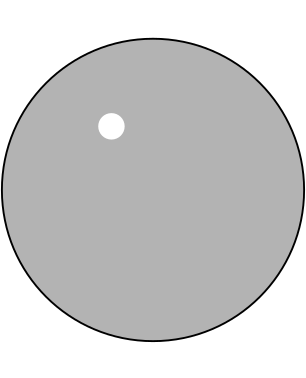}
            \vspace{5mm}
        \end{subfigure}
        \begin{subfigure}{0.7\textwidth}
            \includegraphics[height=0.27\globaltextwidth]{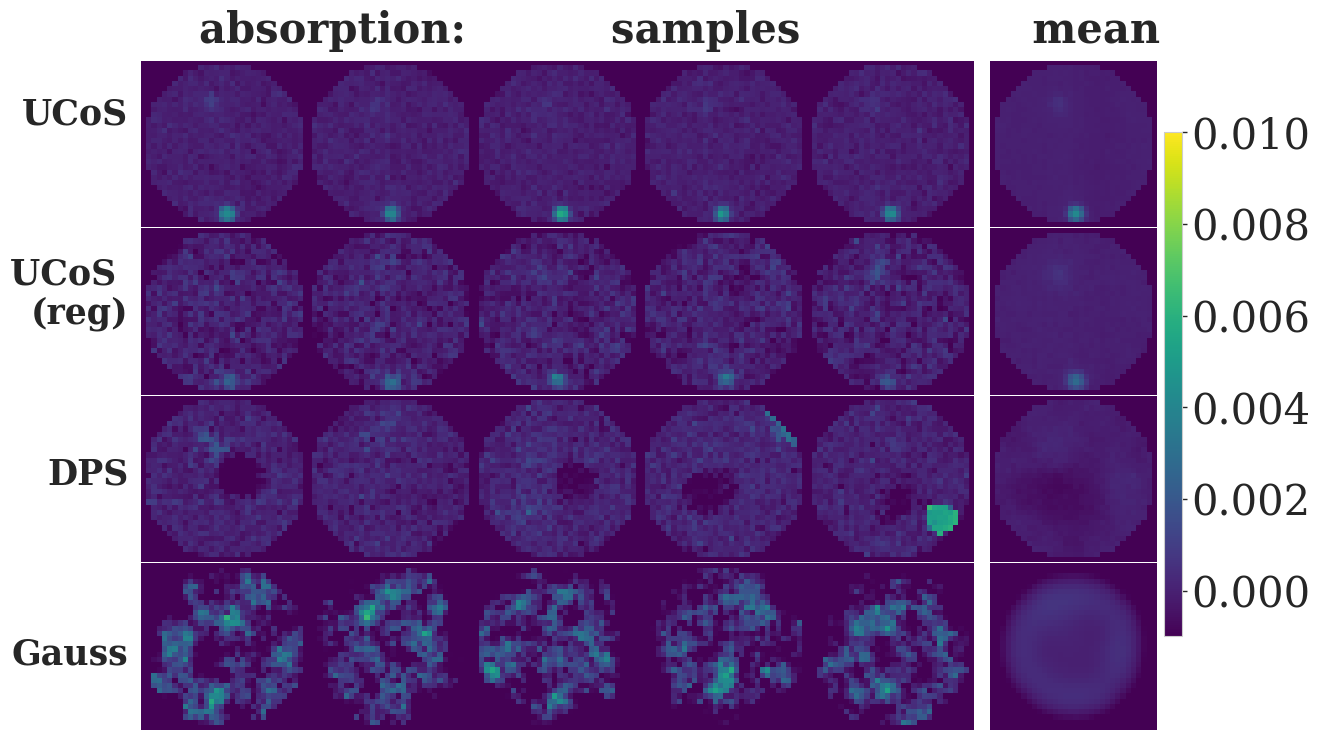}
        \end{subfigure}
        \begin{subfigure}{0.18\textwidth}
        \includegraphics[height=0.27\globaltextwidth]{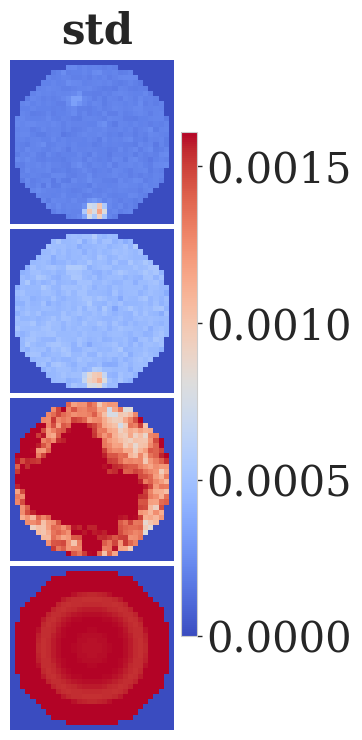}
        \end{subfigure}
        \begin{subfigure}{0.1\textwidth}
            \includegraphics[width=\linewidth]{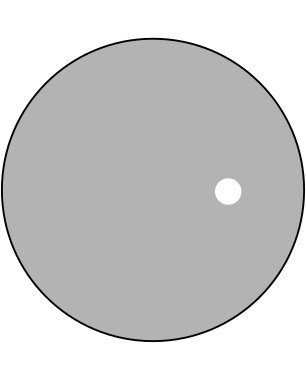}
            \vspace{5mm}
        \end{subfigure}
        \begin{subfigure}{0.7\textwidth}
            \includegraphics[height=0.27\globaltextwidth]{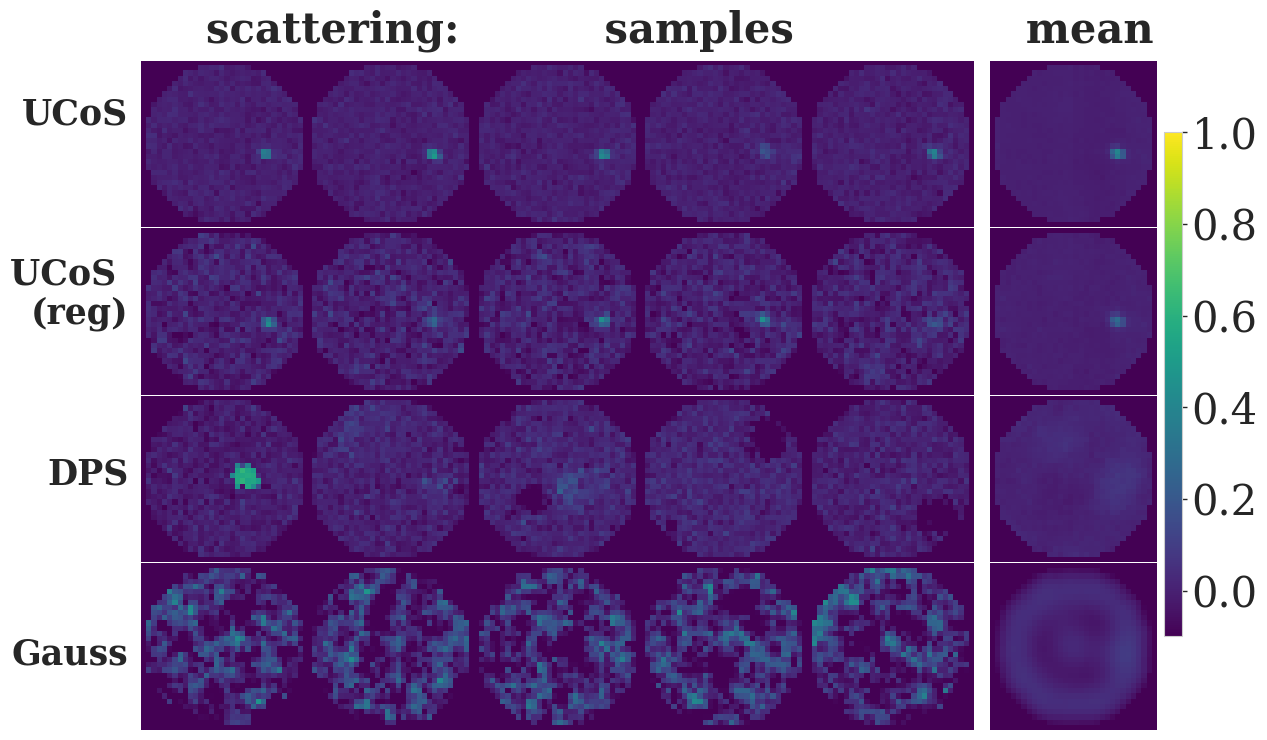}
        \end{subfigure}
        \begin{subfigure}{0.18\textwidth}
            \includegraphics[height=0.27\globaltextwidth]{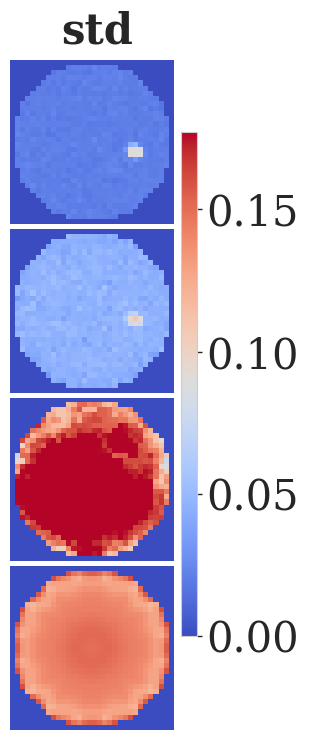}
        \end{subfigure}
        \caption{Intralipid x2, ink x2}
        \vfill
    \end{subfigure}
    \begin{subfigure}{0.8\textwidth}
            \begin{subfigure}{0.1\textwidth}
            \includegraphics[width=\linewidth]{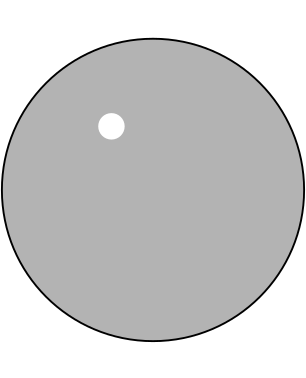}
            \vspace{5mm}
        \end{subfigure}
        \begin{subfigure}{0.7\textwidth}
            \includegraphics[height=0.27\globaltextwidth]{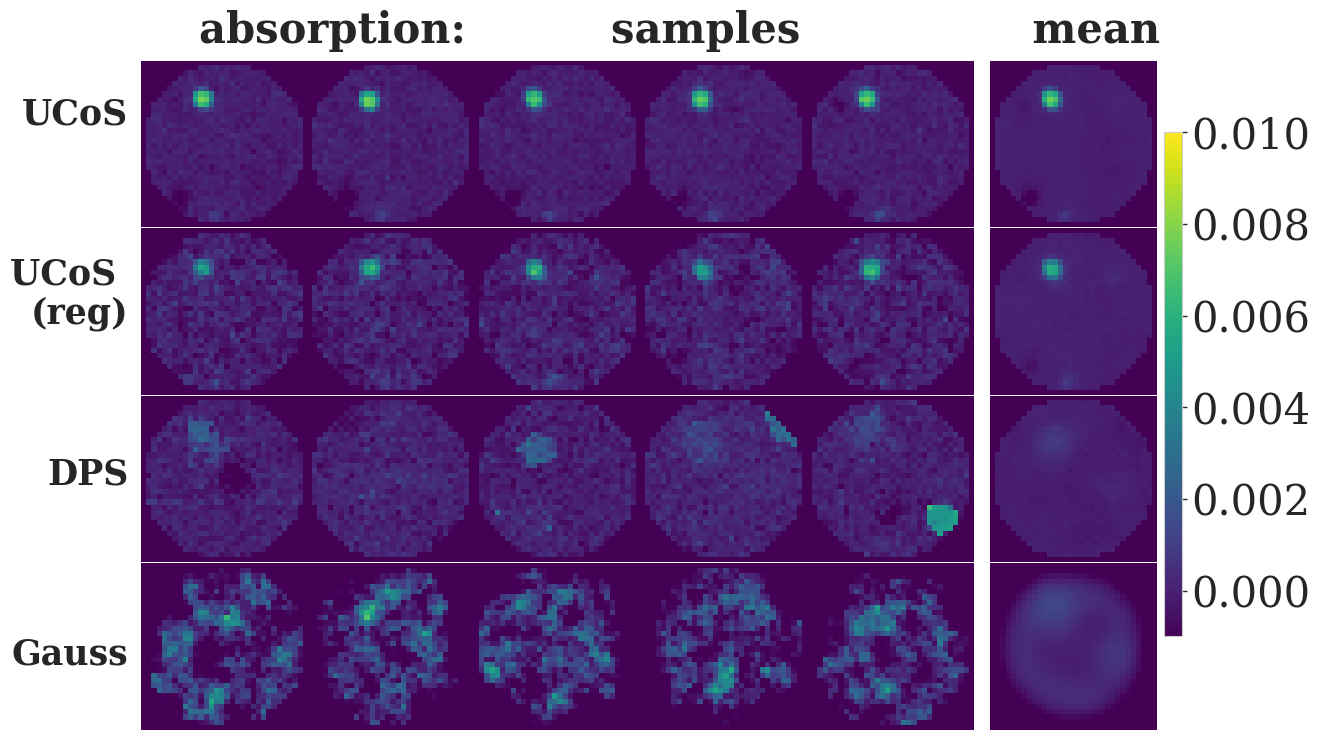}
        \end{subfigure}
        \begin{subfigure}{0.18\textwidth}
            \includegraphics[height=0.27\globaltextwidth]{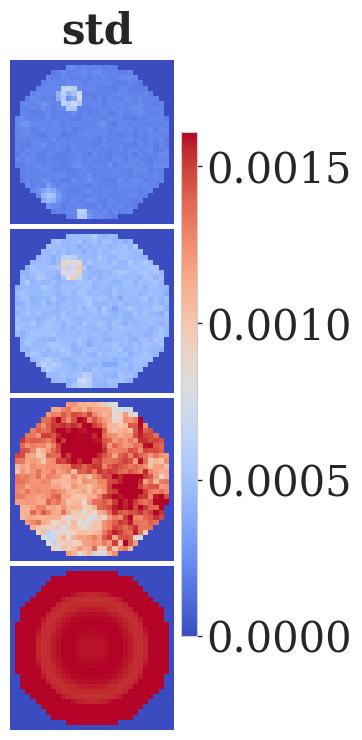}
        \end{subfigure}
        \begin{subfigure}{0.1\textwidth}
            \includegraphics[width=\linewidth]{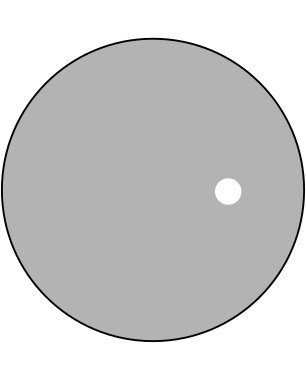}
            \vspace{5mm}
        \end{subfigure}
        \begin{subfigure}{0.7\textwidth}
        \includegraphics[height=0.27\globaltextwidth]{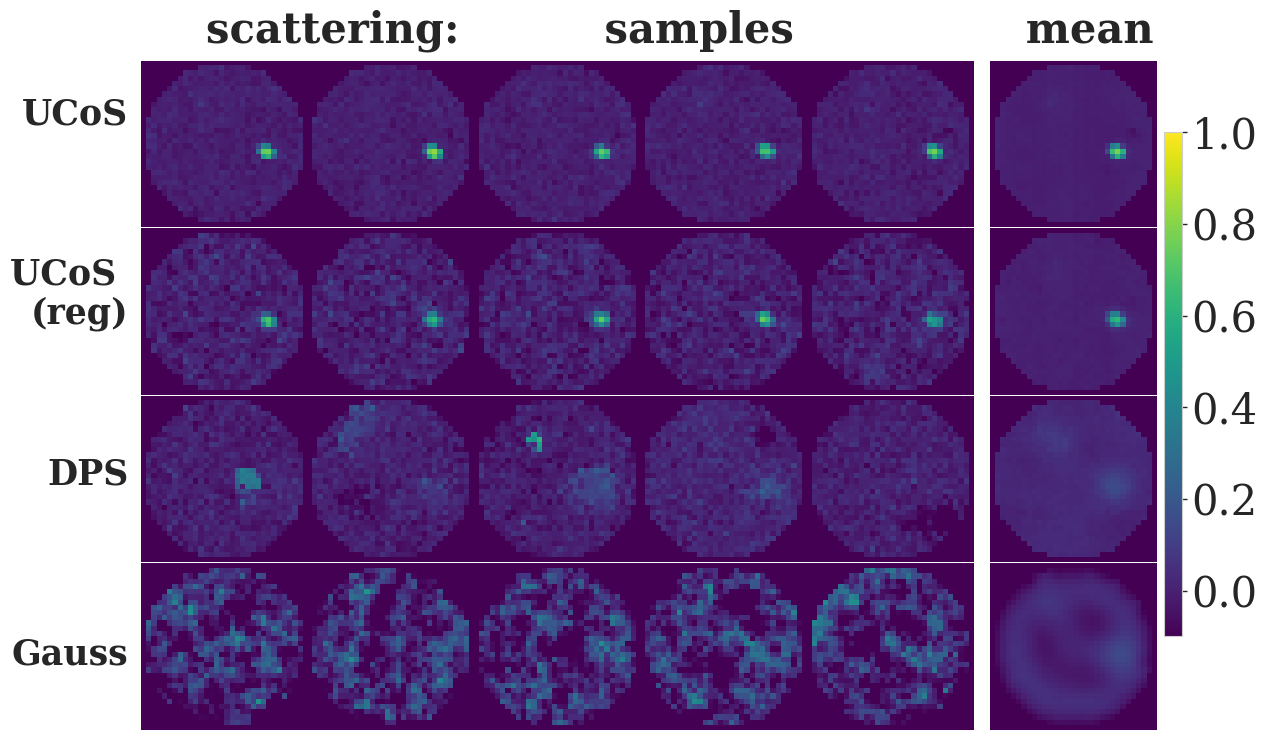}
        \end{subfigure}
        \begin{subfigure}{0.18\textwidth}
            \includegraphics[height=0.27\globaltextwidth]{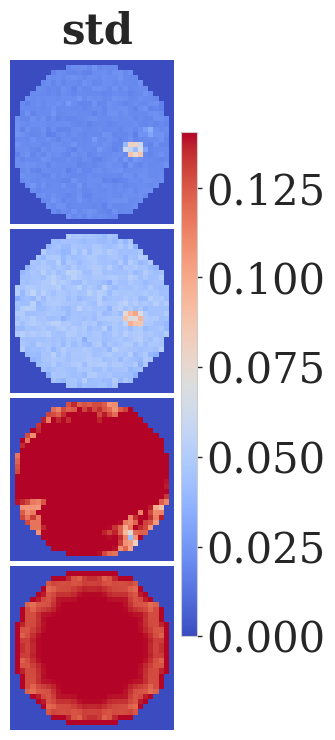}
        \end{subfigure}
        \caption{Intralipid x5, ink x5}
    \end{subfigure}
    \caption{Reconstructions using the experimental data measured in phantoms with two inclusions containing: (a) two times increased concentration of ink (absorption inclusion) or intralipid (scattering inclusion) as compared to the background phantom; (b) same as (a) but with five times increased concentrations.
    Rows in each panel show different reconstruction methods listed from top to bottom: UCoS, regularized UCoS, DPS, Gaussian OU process prior.
    Each row presents from left to right: approximate ground truth, 5 samples, sample average and standard deviation.}
    \label{fig:expt1}
\end{figure}

\begin{figure}
    \centering
    \begin{subfigure}{0.8\textwidth}
        \begin{subfigure}{0.1\textwidth}
            \includegraphics[width=\linewidth]{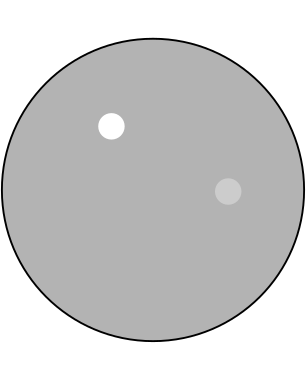}
            \vspace{5mm}
        \end{subfigure}
        \begin{subfigure}{0.7\textwidth}
            \includegraphics[height=0.27\globaltextwidth]{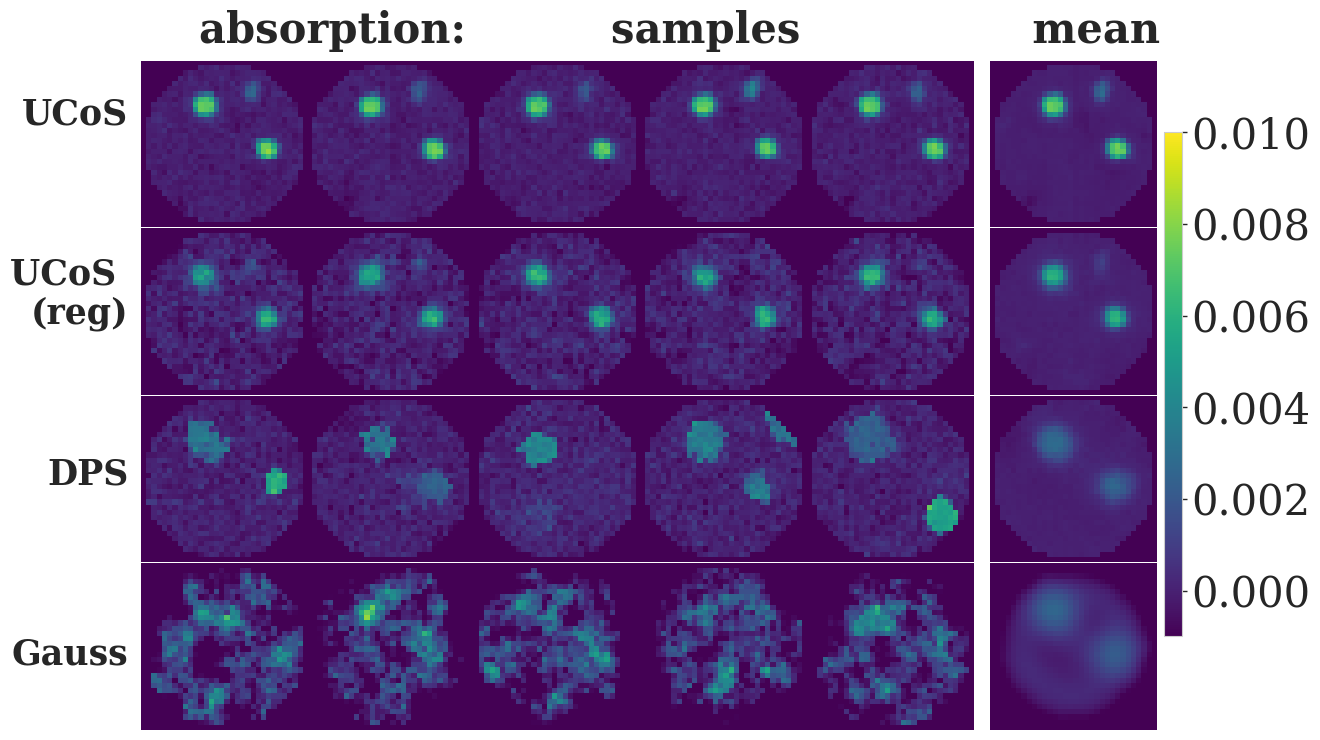}
        \end{subfigure}
        \begin{subfigure}{0.18\textwidth}
            \includegraphics[height=0.27\globaltextwidth]{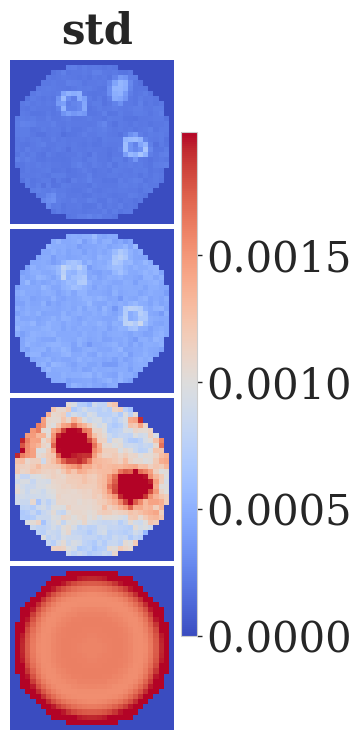}
        \end{subfigure}
        \hfill
        \begin{subfigure}{0.1\textwidth}
            \includegraphics[width=\linewidth]{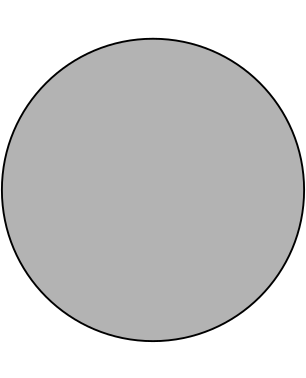}
            \vspace{5mm}
        \end{subfigure}
        \begin{subfigure}{0.7\textwidth}
        \includegraphics[height=0.27\globaltextwidth]{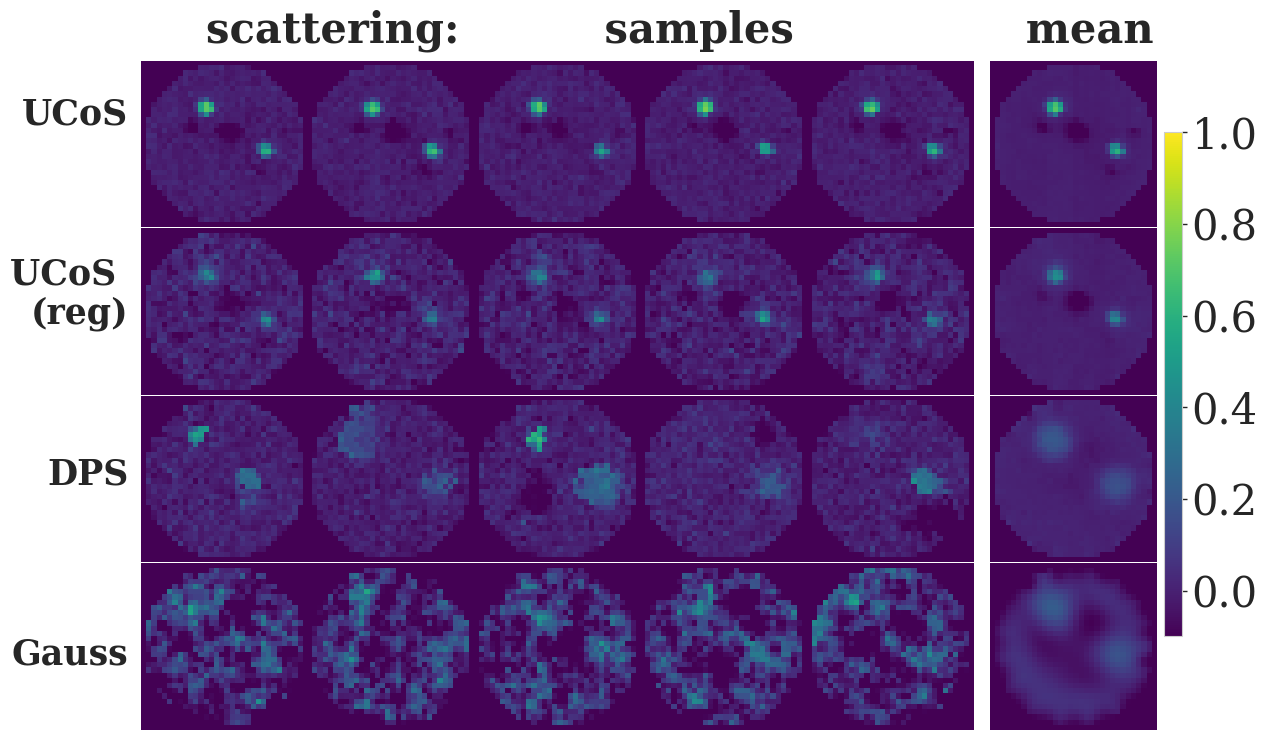}
        \end{subfigure}
        \begin{subfigure}{0.18\textwidth}
            \includegraphics[height=0.27\globaltextwidth]{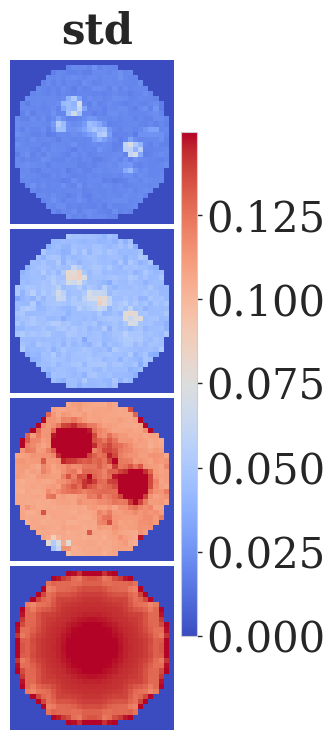}
        \end{subfigure}
        \caption{ICG 1\,\textmu g/ml, 1.5\,\textmu g/ml}
    \end{subfigure}
    \begin{subfigure}{0.8\textwidth}
                \begin{subfigure}{0.1\textwidth}
            \includegraphics[width=\linewidth]{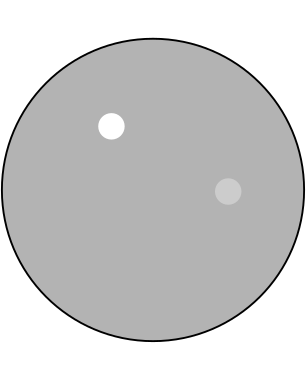}
            \vspace{5mm}
        \end{subfigure}
        \begin{subfigure}{0.7\textwidth}
            \includegraphics[height=0.27\globaltextwidth]{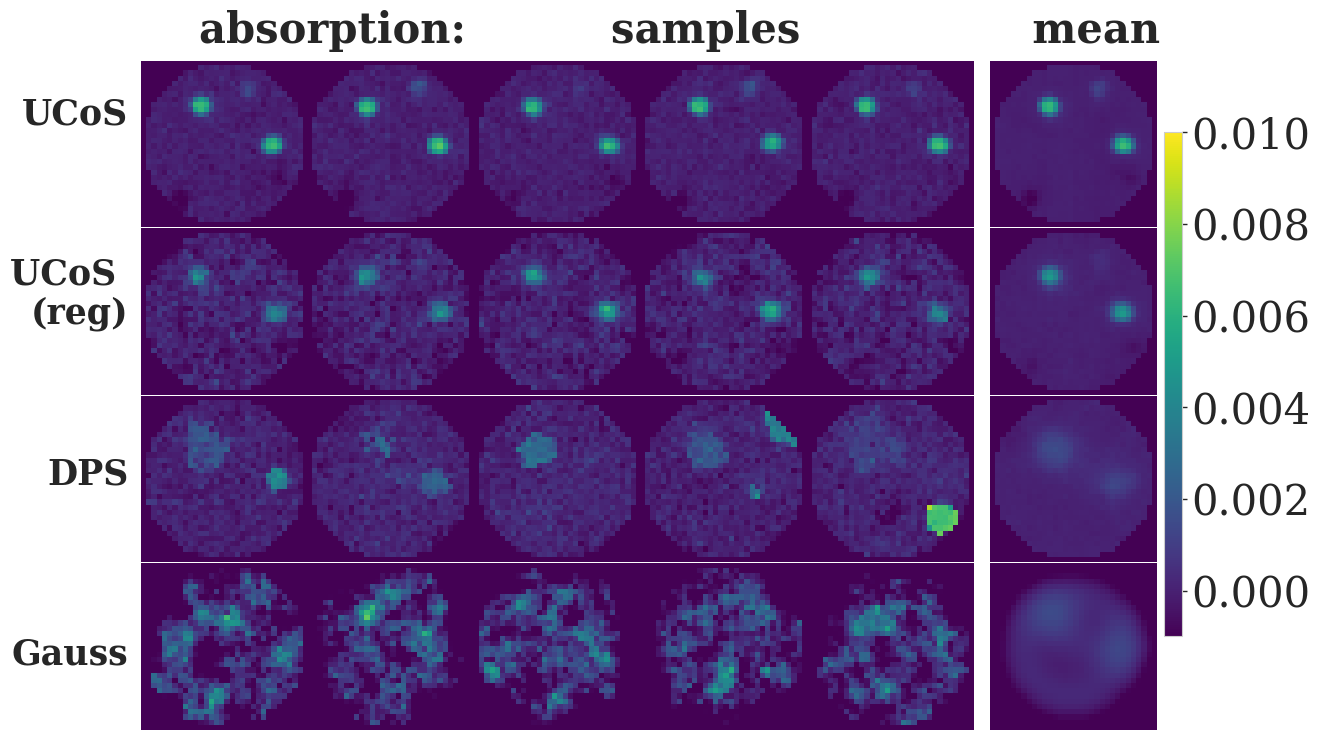}
        \end{subfigure}
        \begin{subfigure}{0.18\textwidth}
            \includegraphics[height=0.27\globaltextwidth]{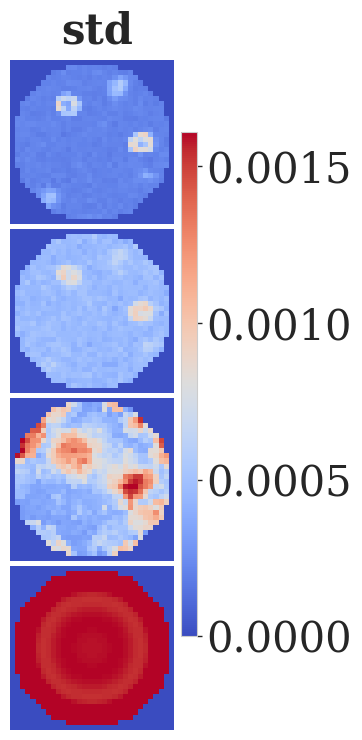}
        \end{subfigure}
                \begin{subfigure}{0.1\textwidth}
            \includegraphics[width=\linewidth]{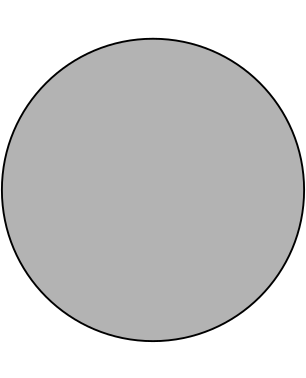}
            \vspace{5mm}
        \end{subfigure}
        \begin{subfigure}{0.7\textwidth}
        \includegraphics[height=0.27\globaltextwidth]{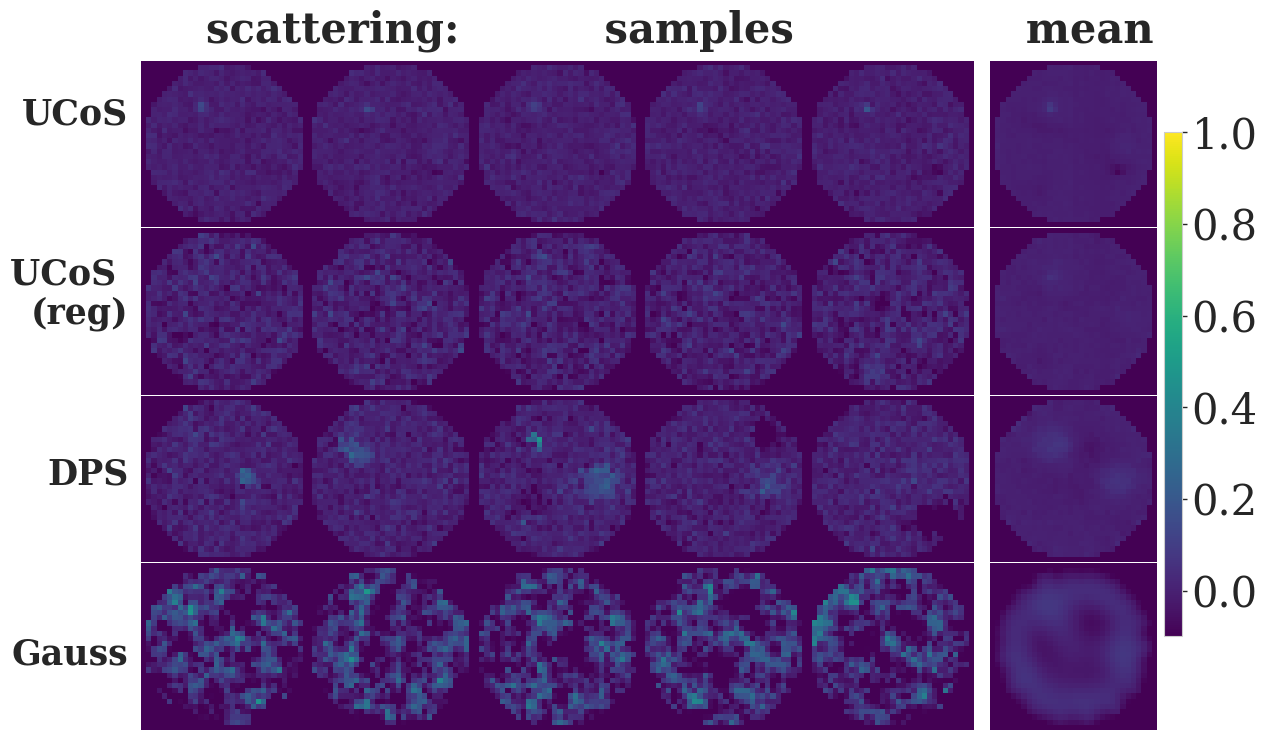}
        \end{subfigure}
        \begin{subfigure}{0.18\textwidth}
            \includegraphics[height=0.27\globaltextwidth]{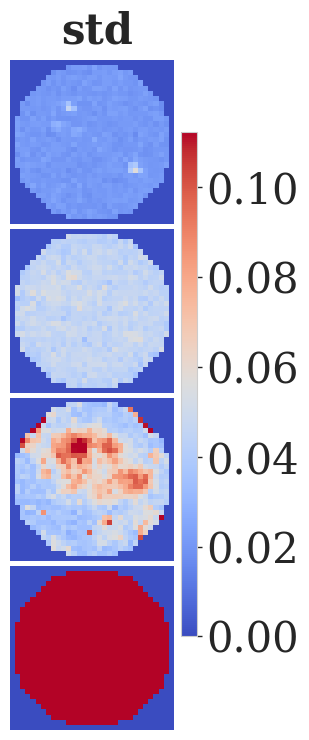}
        \end{subfigure}
        \caption{BPSi NPs 20\,\textmu g/ml, 30\,\textmu g/ml}
    \end{subfigure}
    \caption{Reconstructions using the experimental data measured in phantoms with two inclusions containing: (a) ICG 1\,\textmu g/ml and 1.5\,\textmu g/ml; (b) BPSi NPs 20\,\textmu g/ml and 30\,\textmu g/ml.
    Rows in each panel show different reconstruction methods listed from top to bottom: UCoS, regularized UCoS, DPS, Gaussian OU process prior.
    Each row presents from left to right: approximate ground truth, 5 samples, sample average and standard deviation.}
    \label{fig:expt2}
\end{figure}

\FloatBarrier

\section{Discussion and Conclusions}
This work studies SBD models for posterior sampling in the severely ill-posed DOT difference imaging problem. We consider experimental data which is naturally subject to approximation errors and conduct simulations with full- and limited-view geometries. 
We introduce a novel regularization of SBD models based on a convex combination of two score functions, such as a model-based and a data-driven score function. Its behaviour is analysed from a theoretical perspective by proving that the convex combination, approximately in the small diffusion time setting, corresponds to the score function of the geometric mixture of the underlying distributions. 

In a numerical study, a classical model-based approach, the approximate diffusion-based method DPS, and the UCoS framework as an exact posterior sampling method are compared. 
The UCoS-based posterior sampling provides the most reliable reconstructions among the considered methods in terms of recovering inclusion locations and amplitudes under full-view, limited-view, and experimentally acquired measurements.
In contrast, DPS and classical model-based reconstructions frequently exhibit missing or spurious inclusions or even structural anomalies, particularly under limited-view conditions and for experimental data.
This demonstrates the importance of exact posterior sampling for maintaining reconstruction fidelity under realistic measurement conditions in the DOT problem.
The introduced regularized variant further improves performance in challenging scenarios, particularly under limited-view geometries and distribution shifts, while maintaining competitive performance in experimental settings.

\fabian{The SBD models were compared against a Gaussian prior, which remains a widely used regularization approach in DOT \cite{hoshi2016overview, Aspri2024}. Its main advantage is analytical tractability, which enabled a direct comparison of the posterior distribution. More expressive priors, such as total variation or sparsity-promoting priors, are predominantly used for MAP estimation and are not usually amenable to full posterior sampling, which is the primary objective of this work. While such priors would likely reduce the performance gap, we expect the advantages of UCoS to persist, especially in the limited-view and experimental settings, where model mismatch and severely reduced measurement information make the learned prior particularly beneficial. In the full-view setting, a TV-regularized MAP estimate yielded an SSIM of 0.7869, outperforming the Gaussian prior but remaining clearly below all SBD methods.}

Overall, the study highlights the limitations of approximate diffusion-based sampling for severely ill-posed problems and under model mismatch and demonstrates the advantages of UCoS-based methods for stable and accurate reconstruction in realistic DOT inverse problems. 
\fabian{Nonetheless, we note several limitations of this study. The linearized forward model, while widely used and generally robust, introduces its own approximation error, whose interaction with SBD-based sampling has not been characterized.}

\fabian{Several directions remain for future work. On the application side, extensions to 3D geometries, training with experimentally measured noise statistics, and incorporation of source and detector calibration uncertainties are open challenges. On the theoretical side, the geometric mixture interpretation is currently restricted to finite dimensions and to the small diffusion-time regime. Extending the geometric mixtures to infinite dimensions and studying the approximation of the corresponding score in infinite dimensions remains a direction for future work. }

\section*{Acknowledgements}
LT acknowledges support from the Austrian Science Fund (FWF), Elise-Richter grant DOI 10.55776/V1000. Part of this research was performed while DLD was visiting the Institute for Mathematical and Statistical Innovation (IMSI), which is supported by the National Science Foundation (Grant No. DMS-2425650). Funding was also received from the European Research Council (ERC) under the Horizon 2020 research and innovation program of the European Union (Grant agreement No. 101125225 and 101001417) and the Research Council of Finland through the Flagship of Advanced Mathematics for Sensing Imaging and Modelling (Decision Numbers 358944 and 359183), Centre of Excellence in Inverse Modelling and Imaging (Decision Numbers 353086 and 353094) and project grants (Decision Number 348504). Support by the Austrian Science Fund (FWF) grant I6667-N and the Magnus Ehrnrooth Foundation is acknowledged.

\printbibliography

\end{document}